\documentclass{article}

\usepackage[textsize=tiny]{todonotes}
\usepackage[final]{changes}

\usepackage[preprint, nonatbib]{neurips_2023}

\usepackage[utf8]{inputenc} 
\usepackage[T1]{fontenc}    
\usepackage{hyperref}       
\usepackage{url}            
\usepackage{booktabs}       
\usepackage{amsfonts}       
\usepackage{nicefrac}       
\usepackage{microtype}      
\usepackage{xcolor}         
\usepackage{tikz}           
\usepackage{circuitikz}
\usetikzlibrary{shadows}
\usetikzlibrary{shadows.blur}

\usepackage{graphicx}
\usepackage{common}
\usepackage{caption}
\usepackage{subcaption}
\usepackage{color}
\usepackage{tabularray}

\usepackage[capitalize,noabbrev]{cleveref}

\newcommand{\stepsize}{\rho}
\newcommand{\ri}{\mathrm{relint}\, }

\title{Bringing regularized optimal transport to lightspeed: a splitting method adapted for GPUs}

\author{
 Jacob Lindbäck \\
  EECS, KTH\\
  Stockholm, Sweden\\
  \texttt{jlindbac@kth.se} 
  \And
  Zesen Wang \\
  EECS, KTH\\
  Stockholm, Sweden\\
  \texttt{zesen@kth.se}
  \And
  Mikael Johansson \\
  EECS, KTH\\
  Stockholm, Sweden\\
  \texttt{mikaelj@kth.se}
}

\begin{document}

\maketitle

\begin{abstract}
    We present an efficient algorithm for regularized optimal transport. In contrast to previous methods, we use the Douglas-Rachford splitting technique to develop an efficient solver that can handle a broad class of regularizers. The algorithm has strong global convergence guarantees, low per-iteration cost, and can exploit GPU parallelization, making it considerably faster than the state-of-the-art for many problems. We illustrate its competitiveness in several applications, including domain adaptation and learning of generative models.
\end{abstract}

\section{Introduction}

Optimal transport (OT) is an increasingly important tool for 
many ML problems. It has proven successful in a wide range of applications, including domain adaptation \cite{courty2017joint, gu2022keypoint}, learning generative models \cite{arjovsky2017wasserstein, genevay2018learning}, smooth ranking and sorting schemes \cite{NEURIPS2019_d8c24ca8}, and long-tailed recognition \cite{peng2022optimal}. The versatility of OT stems from the fact that it provides a flexible framework for comparing probability distributions that incorporates the geometric structure of the underlying sample space \cite{peyre2019computational}. Early on, practitioners were directed to LP solvers with poor scalability to solve OT problems, but this changed dramatically with the introduction of entropic regularization and the Sinkhorn algorithm for OT. Sinkhorn's simple parallelizable operations resolved a computational bottleneck of OT and 
enabled the solution of very large problems at ``lightspeed'' \cite{cuturi2013sinkhorn}. Despite the computational advantages of entropic regularization, many applications depend on sparse or structured transportation plans, and the bias entropic regularization introduces can significantly impact the performance of the downstream task being solved. In such settings, other structure-promoting regularizers are typically preferred (see e.g. \cite{courty2016optimal, liu2022sparsity, gu2022keypoint}). However, to the best of our knowledge, no framework exists 
that handles general regularizers for OT in a unifying way while enjoying the same computational properties as Sinkhorn. To this end, we study OT for a greater class of regularizers and develop an algorithm with similar computational performance as Sinkhorn. Further, we benchmark our method against the state-of-the-art, which validates its competitiveness in several applications.

For discrete probability distributions, solving regularized OT problems amounts to finding a solution to the optimization problem
\begin{align}\label{eq:ot:problem}
    \begin{array}{ll}
        \underset{X \in \R_+^{m \times n}}{\minimize} & \InP{C}{X} + h(X)\\
        \subjectto & X \ones_n = p, \,\, X^\top \ones_m = q.
    \end{array}
\end{align}
Here $p$ and $q$ are non-negative vectors that sum to $1$, $C$ is a non-negative cost matrix, and the regularizer $h$ is a function that promotes structured solutions. For any regularization, in order to solve \eqref{eq:ot:problem} fast, it is important that the large number of decision variables and the non-negativity constraint on the transportation plan are handled in a memory-efficient way. Further, it is crucial to manage the non-smoothness induced by the non-negativity constraints without altering the complexity of the algorithm. It is a common practice to solve the OT problem by considering its dual since many regularization terms give rise to dual problems with structure that can be exploited. Most notably, entropic regularization, i.e. $h(X) = \epsilon \sum_{ij} X_{ij} \log X_{ij}$, where $\epsilon>0$, enables deriving a simple alternating dual ascent scheme, which coincides with the well-known Sinkhorn-Knopp algorithm for doubly stochastic matrices \cite{sinkhorn1967concerning}. An advantage of this scheme is that it is easy to parallelize and has a low per-iteration cost \cite{cuturi2013sinkhorn}. However, choosing the regularization parameter $\epsilon$ too small renders the algorithm numerically unstable, which is particularly problematic when low-precision arithmetic is used \cite{mai2021fast}. Conversely, increasing $\epsilon$ makes the transportation plan blurrier - which can be problematic in applications when sparse solutions are desired \cite{blondel2018smooth}. 

For more general regularizers, it is often difficult to develop fast algorithms using this technique. However, if the regularizer is strongly convex, the dual (or semi-dual) will be smooth, and standard gradient-based methods can be used. This is computationally tractable when the gradients of the dual can be computed efficiently \cite{blondel2018smooth}.  Besides potentially expensive gradient computations, just as for Sinkhorn, lower regularization parameters will slow down the algorithm used, and make it unstable. Furthermore, regularizers that are not strongly convex cannot be handled in this framework. Therefore, we propose a different technique to solve \eqref{eq:ot:problem} for a broad class of regularizers, including many non-smooth and non-strongly convex functions. By using the Douglas-Rachford, we derive an algorithm with strong theoretical guarantees, that solves a range of practical problems fast and accurately. In particular, it efficiently handles regularizers that promote sparse transportation plans, in contrast to entropic regularization.

\subsection{Contributions}
We make the following contributions:
\begin{itemize}
\item We adapt the Douglas-Rachford splitting technique to regularized optimal transport, extending the recent DROT algorithm proposed in  \cite{mai2021fast}.
\item Focusing on a broad class of regularizers encompassing quadratic regularization and group lasso as special cases, we demonstrate global convergence guarantees and an accelerated local linear convergence rate. These results extend the available analysis results for related OT solvers and capture the behavior of the iterates observed in practice.
 
\item We develop an efficient GPU implementation that produces high-quality optimal transport plans faster than both the conventional Sinkhorn algorithm and the previously published DROT algorithm. We then show how the proposed solver can be used for domain adaption and to produce solutions of better quality up to 100 times faster than the state-of-the-art when implemented on GPU. 
\end{itemize}

\subsection{Related Work}

The Sinkhorn algorithm is arguably the most popular OT solver for the time being,
and it can often 
find approximate solutions fast, even for large-scale OT problems. 
Many extensions and improvements have been proposed, including variations with improved numerical stability~\cite{schmitzer2019stabilized} and memory-efficient versions based on kernel operations \cite{feydy2019interpolating}. However, to our best knowledge, no framework exists for general regularizers that enjoy comparable computational properties. As an example, a standard approach for group-lasso regularization is to linearize the regularization term and use Sinkhorn iteratively \cite{courty2014domain}. Although this approach is fairly generic, it requires that the transportation plan is recovered in every iteration, adding significant computational overhead. For strongly convex regularizers, such as quadratically regularized OT, several dual and semi-dual methods have been proposed, e.g. \cite{blondel2018smooth, lorenz2021quadratically}, with stochastic extensions  \cite{seguy2017large} and non-convex versions to deal with cardinality constraints \cite{liu2022sparsity}. However, these methods are significantly harder to parallelize, rendering them slower than Sinkhorn for larger problems \cite{peyre2019computational}. Moreover, just as for Sinkhorn, both convergence rates and numerical stability properties are significantly worsened for lower regularization parameters.

A promising research direction that has gained traction recently is to consider splitting methods as an alternative to Sinkhorn for OT problems. For instance, an accelerated primal-dual method for OT and Barycenter problems was proposed in \cite{chambolle2022accelerated}, and an algorithm for unregularized OT based on Douglas-Rachford splitting was developed in \cite{mai2021fast}. 
The convergence of splitting methods is well-studied even for general convex problems \cite{bauschke2011convex}. For Douglas-Rachford splitting, tight global linear convergence rates can be derived under additional smoothness and strong convexity assumptions \cite{giselsson2016linear}. Global linear rates have also been established for certain classes of non-smooth and non-strongly convex problems \cite{wang2017new,applegate2022faster}. Douglas-Rachford splitting can also benefit from the inherent sparsity of the problem at hand. For certain classes of problems, these algorithms can identify the correct sparsity pattern of the solution in finitely many iterations, after which a stronger local convergence rate starts to dominate~\cite{liang2017local, poon2019trajectory}. This paper contributes to this line of research, by introducing a splitting method for regularized OT with strong theoretical guarantees and unprecedented computational performance.

\subsection*{Notation}
For any $X, Y \in \R^{m \times n}$, we let $\InP{X}{Y} := \text{tr}(X^\top Y)$ and $\Vert X \Vert_F := \sqrt{\InP{X}{Y}}$. $\R_+^{m \times n}$ and $\R_-^{m \times n}$ denote the set of $m \times n$ matrices with non-negative entries and non-positive entries respectively. The projection onto $\R_+^{m \times n}$ is denoted $[X]_+$, which sets all negative entries of $X$ to zero. The subdifferential of an extended-real valued function $h$ is denoted $\partial h(X)$. We let $\iota_S$ denote the indicator function over a closed convex set $S$, i.e $\iota_S(X) = 0$ if $X\in S$ and $\iota_S(X) = \infty$ if $X \notin S$. The relative interior of a set $S$ is denoted $\ri S$.

\section{Regularized Optimal Transport}

We consider regularized OT problems on the form~\eqref{eq:ot:problem} where the function $h$ is \emph{sparsity promoting} in the sense that its value does not increase if an element of $X$ is set to zero. 

\begin{definition}[Sparsity promoting regularizers]\label{def:h:sparsity} {\ }
  $h: \R^{m \times n} \to \R \cup \{+\infty \} $ is said to be sparsity promoting if, for any $X \in \R^{m \times n}$, $h(X) \geq h(X_s)$ for every $X_s \in \R^{m \times n}$ with $(X_s)_{ij} \in \{0,\, X_{ij}\}$.
\end{definition}
We will also assume that $h$ is closed, convex, and proper over the feasible set of \eqref{eq:ot:problem}.

Notice that this function class does not necessarily induce sparsity. For instance $h = 0$, or $h= \Vert \cdot \Vert_F^2$ meet the conditions of Definition~\ref{def:h:sparsity}. Besides these two examples, the class of sparsity promoting functions include, but are not limited to, the following functions. 

\begin{itemize}
    \item $h(X) = \lambda \sum_{g \in \mathcal{G}} \Vert X_g \Vert_F $ (group lasso OT)
    \item $h(X) = \sum_{ij} X_{ij} \mathrm{arcsinh}({X_{ij}/\beta }) - (X_{ij}^2 + \beta^2)^{1/2}$ (hypentropic regularization)
    \item $h(X) =  \sum_{ij} w_{ij} \vert X_{ij} \vert$ (weighted $\ell_1$-regularization)
    \item $h(X) = \sum_{(ij) \in \mc{S}} \iota_{X_{ij}=0}(X)$ (constrained OT)
\end{itemize}
Convex combinations of sparsity promoting functions are also sparsity promoting. Moreover, many regularized OT problems can be converted to only involve sparsity promoting regularizers, such as Gini-regularized OT, and regularizers on the form $\Vert X - A \Vert_F^2$ (see e.g. \cite{roberts2017gini}).

In this paper, we will develop a scalable algorithm for solving OT problems with sparsity-promoting regularizers. Besides strong theoretical guarantees, the algorithm is  numerically stable and can be efficiently implemented on a GPU. Our approach builds on the recently proposed algorithm for unregularized OT, which the authors refer to as DROT \cite{mai2021fast}. We extend the DROT algorithm to regularized OT problems, improve the theoretical convergence guarantees, and develop a faster GPU kernel. As a result, our algorithm allows us to solve regularized OT problems faster, and more accurately in a more generic fashion compared to currently available solvers. To derive the algorithm, we review the splitting method for OT developed in \cite{mai2021fast}.
\subsection{Douglas-Rachford splitting for OT}
Douglas-Rachford splitting is 
a technique for solving optimization problems on the form 
\begin{align}\label{eq:composite:problem}
    \minimize_{x\in \R^n} \; f(x) + g(x)
\end{align}
using the fixed point update $y_{k+1} = T(y_k)$, where
\begin{align}\label{eq:dr:operator}
     T(y) = y + \prox{\stepsize g}{2 \prox{\stepsize f} {y} - y}  - \prox{\stepsize f} {y}
\end{align}
and $\stepsize>0$ is a stepsize parameter. If $y^\star$ is a fixed point of $T$ then $x^\star = \prox{\stepsize f}{y^\star}$ is a solution to~\eqref{eq:composite:problem}. Often a third iterate $z_k$ is introduced, 
and the Douglas-Rachford iterations are expressed as

\begin{align}\label{eq:dr:three:updates}
    x_{k+1} = \prox{\stepsize f}{y_k}, \quad z_{k+1} = \prox{\stepsize g}{2x_{k+1} - y_k} \quad
    y_{k+1} = y_k + z_{k+1} - x_{k+1}. 
\end{align}
As long as $f$ and $g$ are closed and convex (possibly extended-real valued), and an optimal solution exists, the DR-splitting method converges to a solution (see, e.g. Theorem~25.6 in \cite{bauschke2011convex}). 
Under additional assumptions of $f$ and $g$ in~\eqref{eq:composite:problem}, the method has even stronger convergence guarantees~\cite{giselsson2016linear}. For a given problem that can be formulated on the form \eqref{eq:composite:problem}, there are typically many ways to partition the problem between $f$ and $g$. This must be done with care: a poor split can result in an algorithm with a higher per-iteration cost than necessary, or one which requires more iterations to converge than otherwise needed. Often it is difficult to achieve both simultaneously, which poses a challenging trade-off between the iterate complexity and the per-iteration cost of the resulting algorithm.

To facilitate the derivation of the algorithm, we let $\mc{X} = \{X \in \R^{m \times n}: X \ones_n = p, \, X^\top \ones_m = q \}$ and introduce two indicator functions $\iota_{\R_+^{m \times n}}$ and $\iota_{\mathcal{X}}$, that correspond to the constraints of \eqref{eq:ot:problem}. 
Recognizing that the projection onto $\mc{X}$ is easy to compute, the authors of~\cite{mai2021fast}  proposed the splitting 
\begin{align*}
f(X) = \InP{C}{X} + \iota_{\R_+^{m \times n}}(X) \;\;\mbox{  and  }\;\; g(X)= \iota_{\mathcal{X}}(X)\
\end{align*}
and demonstrated that 
the update \eqref{eq:dr:operator} can 
be simplified to 
\begin{align}\label{eq:simplified:update}
    X_{k+1} = [Y_k-\stepsize C]_+, \quad Y_{k+1} = X_{k+1} + \phi_{k+1} \ones_n^\top + \ones_m \varphi_{k+1}^\top 
\end{align}
where $\phi_k$ and $\varphi_k$ are vectors given via the iterative scheme:
\begin{align*}
        r_{k+1} &= X_{k+1}\ones_n -p, 
        \quad  s_{k+1} = X_{k+1}^\top \ones_m -q, \quad \eta_{k+1} = f^\top r_{k+1} / (m+n), \nonumber\\ \theta_{k+1} &= \theta_k - \eta_{k+1}, \qquad
        a_{k+1} = a_{k} - r_{k+1}, \qquad\, b_{k+1} = b_{k} - s_{k+1}, 
\end{align*}
\begin{align}\label{eq:phi:update}\phi_{k+1} = \left(a_k - 2 r_{k+1} + (2\eta_{k+1} - \theta_k) f\right) / n, \nonumber\\ \varphi_{k+1} = \left(b_k - 2 s_{k+1}  + (2\eta_{k+1} - \theta_k) e\right)/m.
\end{align}
The most expensive operation in~\eqref{eq:phi:update} is the matrix-vector product in the $r$ and $s$ updates. Nonetheless, these operations can efficiently be parallelized on a GPU. Since $Y_k$ can be eliminated, the splitting scheme only needs to update the transportation plan $X_k$ and the ancillary vectors and scalars in \eqref{eq:phi:update}.

\section{DR-splitting for regularized OT}
When extending the DR-splitting scheme to the regularized setting, it is crucial to incorporate $h$ in one of the updates of \eqref{eq:dr:three:updates} in a way that makes the resulting algorithm efficient. Specifically, if we let
\begin{align}\label{eq:f:g}
    f(X) = \InP{C}{X} + \iota_{\R_+^{m \times n}}(X) + h(X), 
 \text{ and } g(X)= \iota_{\mathcal{X}}(X),
 \end{align}
the first update of \eqref{eq:dr:three:updates} reads:
$$ X_{k+1} = \prox{\rho f}{Y_k} = \prox{\rho h + \iota_{\R_+^{m \times n}}}{Y_k - \rho C}.$$
This update is efficient given that $\prox{\rho h + \iota_{\R_+^{m \times n}}}{\cdot}$ is easy to compute. This is indeed the case when sparsity promoting regularizers are considered. We present the result in Lemma~\ref{lemma:composition}. 
\begin{lemma}\label{lemma:composition}
    Let $f(X)= \InP{C}{X} + \iota_{\R_+^{m \times n}}(X) + h(X)$ where $h(X)$ is sparsity promoting. Then $\prox{\stepsize f}{X} = \prox{\stepsize h}{[X-\stepsize C]_+}$.
\end{lemma}
Lemma~\ref{lemma:composition} enables us to integrate the regularizer into the algorithm with only a minor update. Using the split of \eqref{eq:f:g}, the updated Douglas-Rachford algorithm, which we will refer to as RDROT, follows:
\begin{align}\label{eq:simplified:update:reg}
    X_{k+1} = \prox{\rho h}{[Y_k-\stepsize C]_+}, \quad
    Y_{k+1} = X_{k+1} + \phi_{k+1} \ones_n^\top + \ones_m \varphi_{k+1}^\top.
\end{align}
Under the assumption that $h$ is closed, convex, and proper over the feasible set of \eqref{eq:ot:problem}, 
Theorem~25.6 in~\cite{bauschke2011convex} guarantees that RDROT converges to a solution of~\eqref{eq:ot:problem}. Before we derive even stronger convergence properties of the algorithm, we describe a few specific instances in more detail.

\textbf{Quadratic Regularization} Letting $h(X) = \frac{\alpha}{2} \Vert X \Vert_F^2$ yields 
     $$X_{k+1} = {[Y_k-\stepsize C]_+}/(1+\stepsize \alpha).$$
A nice feature of quadratic regularization is that it has similar limiting properties as entropic regularization when used in OT-divergences \cite{di2020optimal}. But in contrast to entropically regularized OT, letting $\alpha \to 0$ does not cause numerical instability for our method. Furthermore, the quadratic term makes the objective of \eqref{eq:ot:problem} strongly convex, which renders its solution unique. This can be helpful when using OT to define a loss function for e.g. training generative models \cite{genevay2018learning} or adjusting for long-tailed label distributions \cite{peng2022optimal}, since it smoothens the OT cost while preserving sparsity. 

\textbf{Group Lasso Regularization} With $h(X) = \lambda \sum_{g \in \mathcal{G}} \Vert X_g \Vert_F $, where $\mathcal{G}$ is a collection of disjoint index sets, the RDROT update becomes 
\begin{align*}
    \bar{X}_{k+1} = {[Y_k-\stepsize C]_+}, \quad
    X_{k+1, g} = \bigg[1 - \frac{\lambda}{\Vert \bar{X}_{k+1,g} \Vert_F}\bigg]_+ \bar{X}_{k+1,g}, \quad g \in \mc{G}.
\end{align*}
This regularizer has been used extensively used for OT-based domain adaptation \cite{courty2016optimal}. The rationale is that each unlabeled data point in the test (target) domain should only be coupled to data points of the training (source) domain that share the same label. This can be accomplished by organizing the data so that the rows of $X$ correspond to data points in the training domain and the columns to points in the test domain, and then using a sparsity-inducing group-lasso regularizer in \eqref{eq:ot:problem}. In this setting, $\mc{G}$ is the collection of subblocks of the columns that correspond to each label. Consequentially, this regularizer will promote solutions that map each data point in the test domain to a single label.

\subsection{Rate of convergence}\label{subsec:rate:of:conv}

It is well-known that DR-splitting finds an $\epsilon$-accurate solution in $O(1/\epsilon)$ iterations for general convex problems \cite{he20121}. This is a significant improvement over Sinkhorn's iteration complexity of $O(1/\epsilon^2)$ \cite{lin2019efficient}, but the Sinkhorn iterations are very cheap to execute. A key contribution of~\cite{mai2021fast} was the development of a  GPU kernel that executes the DR updates so quickly that the overall method is faster than Sinkhorn for many problem instances. They also demonstrated a global linear convergence by exploiting the geometrical properties of LPs, but the convergence factors are very difficult to quantify or use for step-size tuning. As illustrated in Figure~\ref{fig:conv:sparsity}, and discussed in more detail below, neither the global linear rate nor the $1/k$-rate are globally tight. Instead, the stopping criterion decays as $O(1/k)$ during the initial steps of the RDROT algorithm, but then begins to converge at a fast linear rate. Hence, existing convergence analyses that neglect this local acceleration are unsatisfactory.

\begin{figure*}[!ht]
   \vskip-0.1in
    \centering
    \begin{subfigure}[b]{0.45\textwidth}
        \centering
        {\includegraphics[width=1.\textwidth]{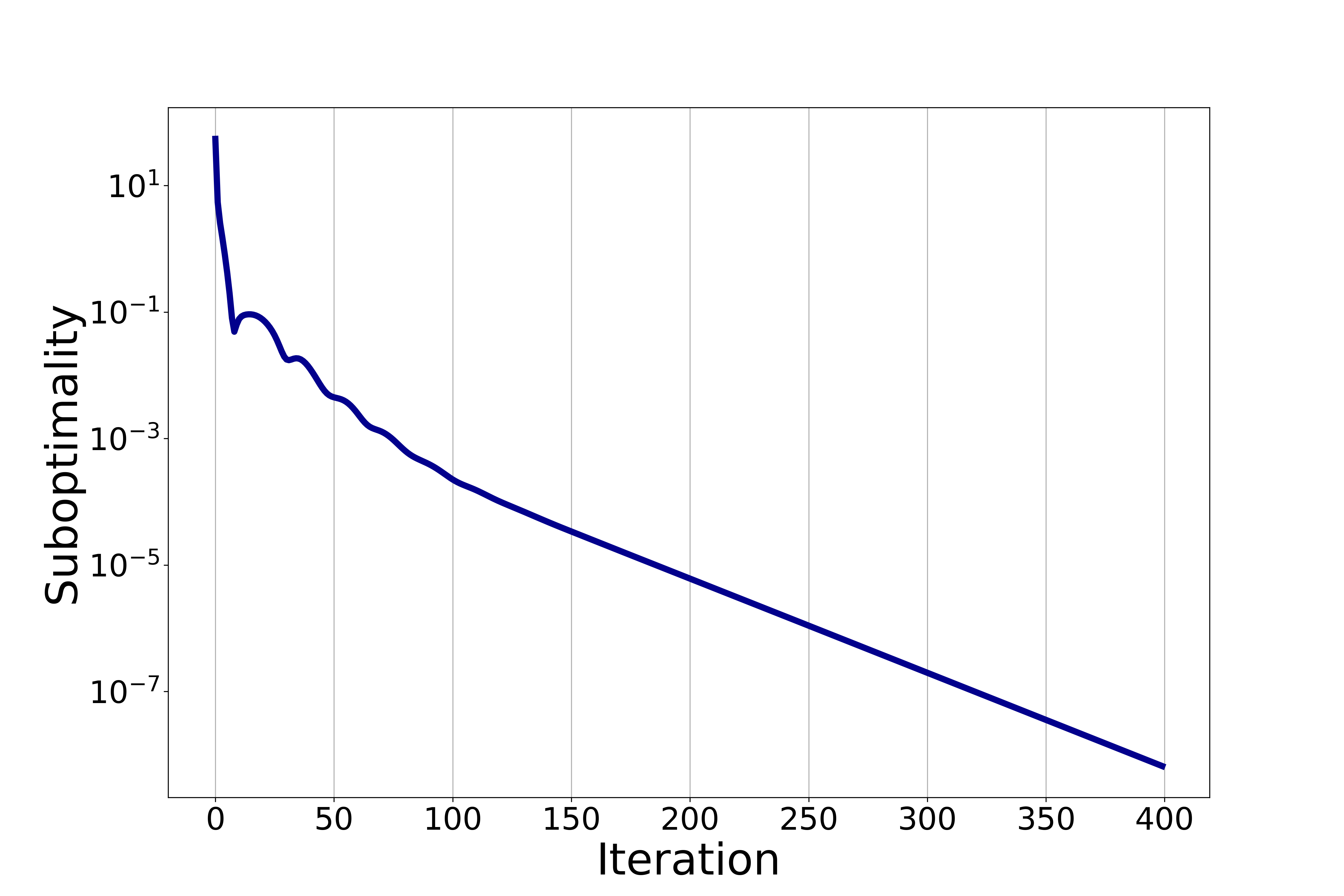}}
    \end{subfigure}
    \hskip-0.22in
    \begin{subfigure}[b]{0.45\textwidth}
        \centering
        {\includegraphics[width=1.\textwidth]{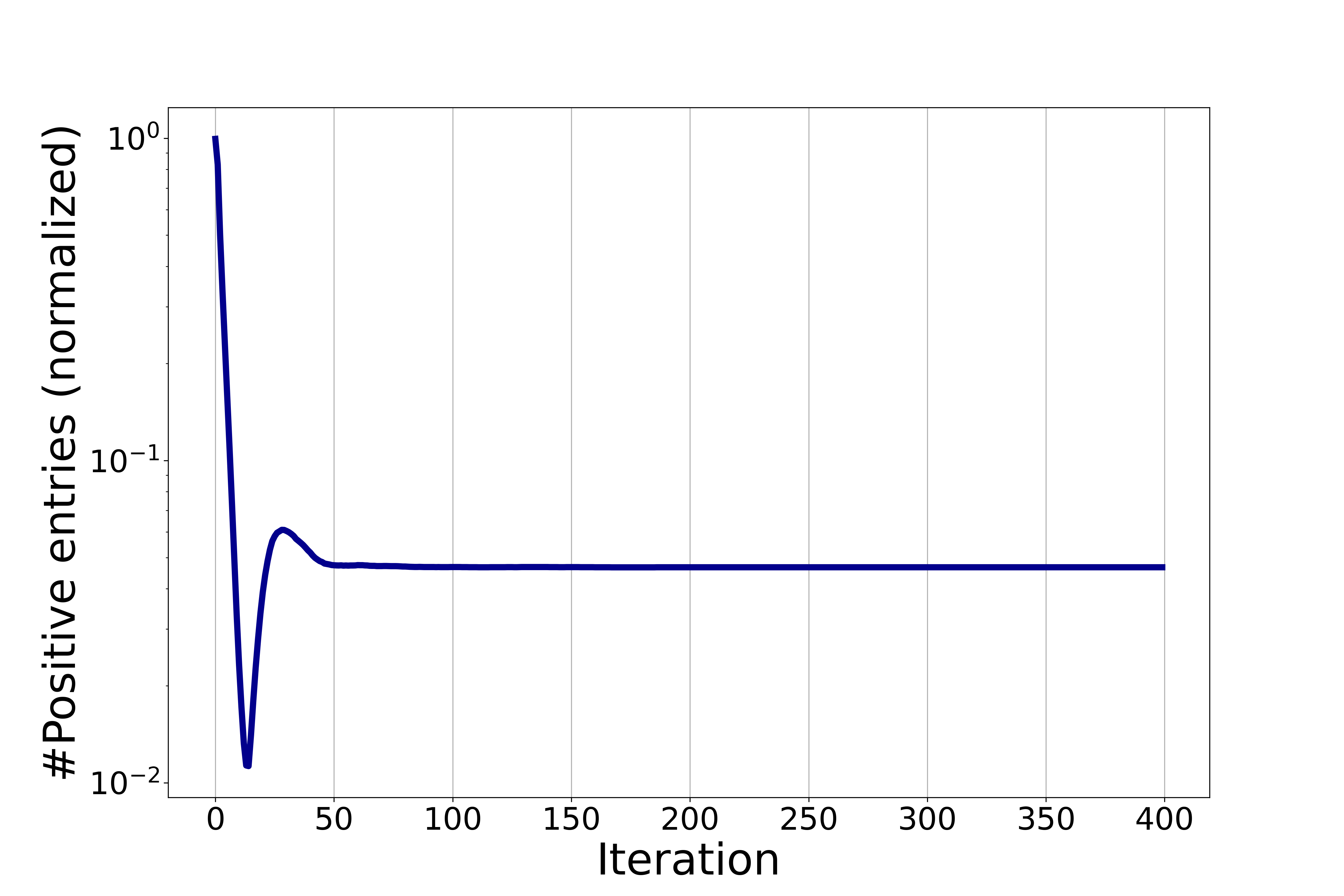}}
    \end{subfigure}
      \caption{An illustration of the convergence behavior of the algorithm for a quadratically regularized problem. The stopping criterion and the number of positive entries were computed in every iteration. Notice that the iteration number at which the sparsity pattern stabilizes is aligned with when the linear convergence starts. In this example, the stepsize is chosen to make this effect more evident.}\label{fig:conv:sparsity}
\end{figure*}

In this section, we will develop a convergence analysis for DR splitting on regularized OT problems that explains this behavior. For many problems, 
a sublinear rate dominates up until a point in which the correct sparsity pattern is identified, and then a fast linear rate starts to dominate. Based on previous work on local convergence properties of DR-splitting \cite{liang2017local}, we establish that our algorithm \eqref{eq:simplified:update:reg} identifies the true sparsity pattern of the solution in finitely many iterations. When the sparsity pattern is identified, a fast linear rate typically starts to dominate over the sublinear rate. 

To derive these guarantees, we need an additional assumption on the optimal solution of \eqref{eq:ot:problem}.

\begin{assumption}\label{assumption:reg:cond}
    Let that $Y^\star$ is a fixed point of \eqref{eq:dr:operator} such that $Y_k \to Y^\star$, and let $X^\star = \prox{\rho f }{Y^\star} = \prox{\rho h}{[Y^\star- \rho C ]_+}$. We then assume that:
\begin{align}\label{eq:reg:cond}
    \frac{1}{\stepsize}\bigg( Y^\ast - X^\ast \bigg) \in  \ri \partial f(X^\star).
\end{align}
\end{assumption}

Assumption~\ref{assumption:reg:cond} can be seen as a stricter version of the traditional optimality condition. It is a common assumption in analyses of active constraint identification and can be traced back to \cite{burke1988identification}. We consider this assumption to be fairly weak since it holds for most relevant OT problems, except for some very specific cost matrices. An extensive discussion is included in the supplementary material.

Under the regularity conditions of \eqref{eq:reg:cond}, we derive the following two convergence results.
\begin{theorem}\label{thm:finite:iter} Let $h(X)$ be sparsity promoting and twice continuously differentiable for $X>0$. If  Assumption~\ref{assumption:reg:cond} holds, then there is a $K \geq 1$ such that for all $k\geq K$, $X_{k, \,ij} = 0$ if and only if $X_{ij}^\star = 0$.
\end{theorem}
For our problem, this result implies that there is a $K$ after which the sparsity pattern does not change. To prove this, we first show that the sparsity-promoting and smoothness assumptions imply that $f$ and $g$ are partly smooth with respect to two manifolds and then invoke results presented in \cite{liang2017local}. Moreover, Theorem~\ref{thm:finite:iter} allows us to derive the following local convergence rate guarantees:
\begin{theorem}
    Assume that the conditions stated in Theorem~\ref{thm:finite:iter} hold. If $h$ is locally polyhedral, then RDROT enjoys a local linear rate: When $k\geq K$, then
    $$ \Vert X_{k+1} - X^\star \Vert_F \leq  r^{k-K}\Vert Y_K - Y^\star \Vert_F, \quad r \in (0,\,1).$$
\end{theorem}
Some examples of which the local polyhedral assumption on $h$ holds include the unregularized case, and the weighted $\ell_1$ regularization: $h(X) = \sum_{ij} a_{ij}\vert X_{ij} \vert$, where $a_{ij} \geq 0$. Further, the optimal rate is always attained independently of the chosen stepsize \cite{liang2017local}. Therefore, one can expect a sublinear rate until the correct sparsity pattern of the transportation plan is identified. From that point onward, a fast linear rate will dominate. In experiments, however, we observe that the local linear rate seems to hold for many more regularizers, although it is, in general, difficult to quantify using this proof technique. A similar observation has also been made in \cite{poon2019trajectory}.

\paragraph{Stepsize selection}
The stepsize implicitly determines how the primal residual descent is balanced against the descent of the duality gap. When the cost has been normalized so that $\Vert C \Vert_\infty = 1$, a stepsize that performs well regardless of $h$ is $\rho = 2(m+n)^{-1}$. This stepsize was proposed for DR-splitting on the unregularized problem in \cite{mai2021fast}. We use this stepsize throughout, and even without any tuning, our method is consistently competitive. 

\paragraph{Initialization}
A reasonable initialization of the algorithm is to let $X_0 = pq^\top$, and $\phi_0 = \mathbf{0}_m$, $\varphi_0 = \mathbf{0}_n$. However, when using the stepsize specified above, this will result in that $X_k = \mathbf{0}_{m \times n}$ for $1 \leq k \leq N$, where $N=O(\max(m,n))$. By initializing $\phi_0 = (3(m+n))^{-1}(1 + m/(m+n)) \ones_m$, and $\varphi_0 = (3(m+n))^{-1}(1 + n/(m+n)) \ones_n$, one skips these first $N$ iterations, resulting in a considerable speed up the algorithm, especially for larger problems. The derivation is added to the supplementary material. There is potential for
further improvements (see e.g. \cite{thornton2023rethinking}), since $\rho^{-1}\phi_k$, $\rho^{-1}\varphi_k$ are related to dual potentials. But for simplicity, we use this strategy throughout the paper.

\paragraph{Stopping criteria}
One drawback of the Sinkhorn algorithm is that its theoretically justified subpotimality bounds are costly to compute from the generated iterates. This is usually addressed by only computing the primal residuals $ \Vert X_k \ones_m -p \Vert$, and $\Vert X_k^\top \ones_n -q \Vert$ every $M$ iterations and terminate when both residuals are less than a chosen tolerance. In contrast, our algorithm computes the primal residuals in every iteration (see $r_k$ and $s_k$ in \eqref{eq:phi:update}), so evaluating the stopping criterion does not add extra computations. Moreover, the iterates $\phi_k$ and $\varphi_k$ will, in the limit, be proportional to optimal dual variables, and can hence be used to estimate the duality gap and the dual residual in every iteration. For more details, we refer to the supplementary material.

\subsection{GPU implementation}

Similar to the approach developed in \cite{mai2021fast}, the RDROT algorithm can be easily parallelized on a GPU, as long as the regularizer is sparsity promoting and its prox-mapping is simple to evaluate. We have developed GPU kernels for RDROT with quadratic and group-lasso regularization, respectively.

For the GPU implementation, the main computation time is spent on updating the $X$ matrix as described in \eqref{eq:simplified:update:reg}. Compared with (unregularized) DROT \cite{mai2021fast}, the RDROT algorithm additionally includes the evaluation of the prox-mappings of the quadratic and group-lasso regularizers.
\begin{itemize}
    \item For the quadratically regularized case, it multiplies each element of $[Y_k-\stepsize C]_+$ with a constant scale to update $X_{k}$. This only requires a minor change and results in practically the same per-iteration-cost as for the unregularized case.
    \item The group-lasso regularizer penalizes the norm of the group, which requires: (1) a reduction for gathering the square of the elements and computing the scale for the group; and (2) a broadcast to apply the scale to $\bar{X}_{k+1,g}$. For small problem sizes, the reduction and the broadcast can be done within a thread block on the shared memory. For large problem sizes, when the elements in one group can not fit in a single thread block, the results of the reduction are stored in global memory, and an additional kernel function is introduced to apply the scale for the group. This overhead increases the per-iteration cost a little, but we show in our experiments that it is still highly competitive compared to other methods. 
\end{itemize}

The remaining parts of the algorithm are for updating vectors, which has much smaller time complexity than the $X$ matrix updates. We detail the implementation in the supplementary material.

\subsection{Backpropagation via Pytorch and Tensorflow wrappers}

In many applications, e.g. \cite{genevay2018learning, salimans2018improving, peng2022optimal}, it is useful to use the solution of \eqref{eq:ot:problem} to construct a loss function. Most notably, its optimal value of \eqref{eq:ot:problem} is directly related to the Wasserstein distance, which is a natural measure of the proximity between distributions that enjoys many advantageous properties (see \cite{peyre2019computational} for an extensive treatment of this topic). When training generative models, for example, the objective is to transform a simple distribution into one that resembles the data distribution. It is thus natural to use the Wasserstein distance to construct loss functions for generative models, see e.g  \cite{arjovsky2017wasserstein}. But in order to backpropagate through such loss functions, one must differentiate the optimal value of \eqref{eq:ot:problem} with respect to the cost matrix $C$ (which is parameterized by both a data batch and a generated batch of samples). Using Fenchel duality, we can show that the optimal solution $X^\star$ is a gradient (or Clarke subgradient) of the loss, regardless of the regularizer. Hence, one can use RDROT for both the forward and backward pass, and do not need to lean on memory-expensive unrolling schemes or computationally expensive implicit differentiation techniques. We include further details on the derivation in the supplementary material. To facilitate using RDROT for DL, we wrapped our GPU OT solvers as PyTorch and TensorFlow extensions. It features fast automatic differentiation and is hence readily used together with any neural architecture that parameterizes a cost matrix.

\section{Numerical experiments}
\subsection*{GPU performance of RDROT with quadratic regularization}
To illustrate the numerical advantages of our algorithm, we compared our GPU implementation of RDROT for quadratically regularized OT (which we will refer to as QDROT) with a
dual ascent method (L-BFGS), proposed in \cite{blondel2018smooth}, which we implemented with PyTorch with all tensors loaded on the GPU. Although there are several other solvers for quadratic OT, such as the semi-dual method proposed in \cite{blondel2018smooth}, and the semi-smooth Newton method from \cite{lorenz2021quadratically}, we have omitted them from our experiments since their associated gradient/search direction computations scale poorly with data size and are difficult to parallelize. In our benchmark, we simulated 50 data sets of size $1000 \times 1000$ and $2000 \times 3000$, respectively. We generated cost matrices with $C_{ij} = \frac{1}{2} \Vert \mathbf{x}_{i} - \mathbf{x}_j \Vert_2^2$, where $\mathbf{x}_i$, $\mathbf{x}_j$ are simulated 2D Gaussian variables with random parameters specific for each data set. We used the quadratic regularization term $\frac{(m+n) \alpha}{2} \Vert X \Vert_F^2$ with $\alpha$ ranging between $0.0005$ to $0.2$. For QDROT, $\alpha=0$ was added for reference. For each dataset and regularization, we ran the algorithms until an accuracy level of $0.0001$ was reached. The results are displayed in Figure~\ref{fig:quad:ot:benchmark}, in which it is clear that our method better exploits the parallelization from the GPU, as it consistently leads to a speedup of at least one order of magnitude compared to the L-BFGS method. Further details and additional experiments with other dataset sizes are included in the supplementary material.  
\begin{figure}[H]
    \centering
    \begin{subfigure}[b]{0.45\textwidth}\label{fig:quad1000}
        \centering
        {\includegraphics[width=1.\textwidth]{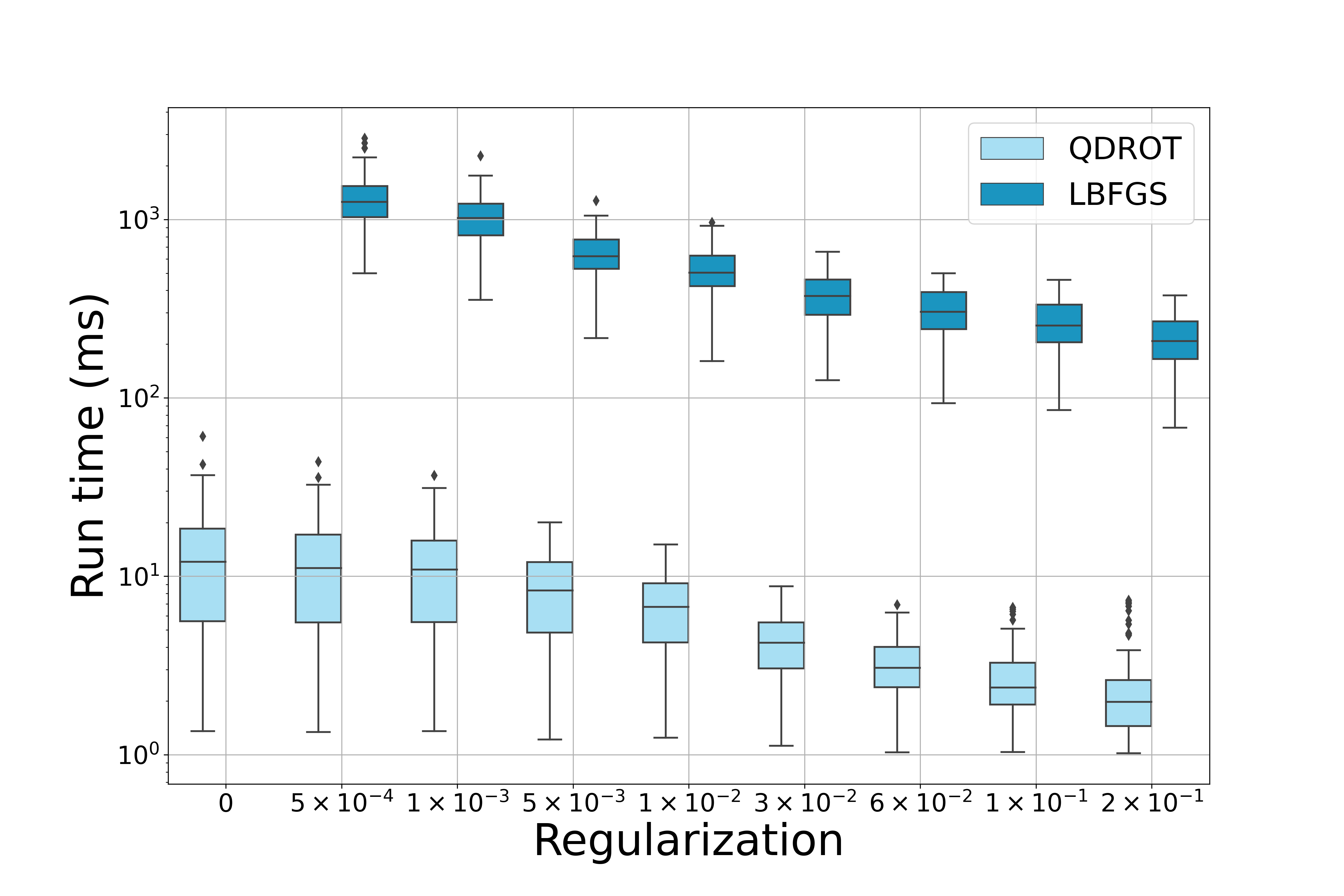}}
        \caption{$m=1000$, $n=1000$}
    \end{subfigure}
    \hskip-0.22in
    \begin{subfigure}[b]{0.45\textwidth}\label{fig:quad3000}
        \centering
        {\includegraphics[width=1.\textwidth]{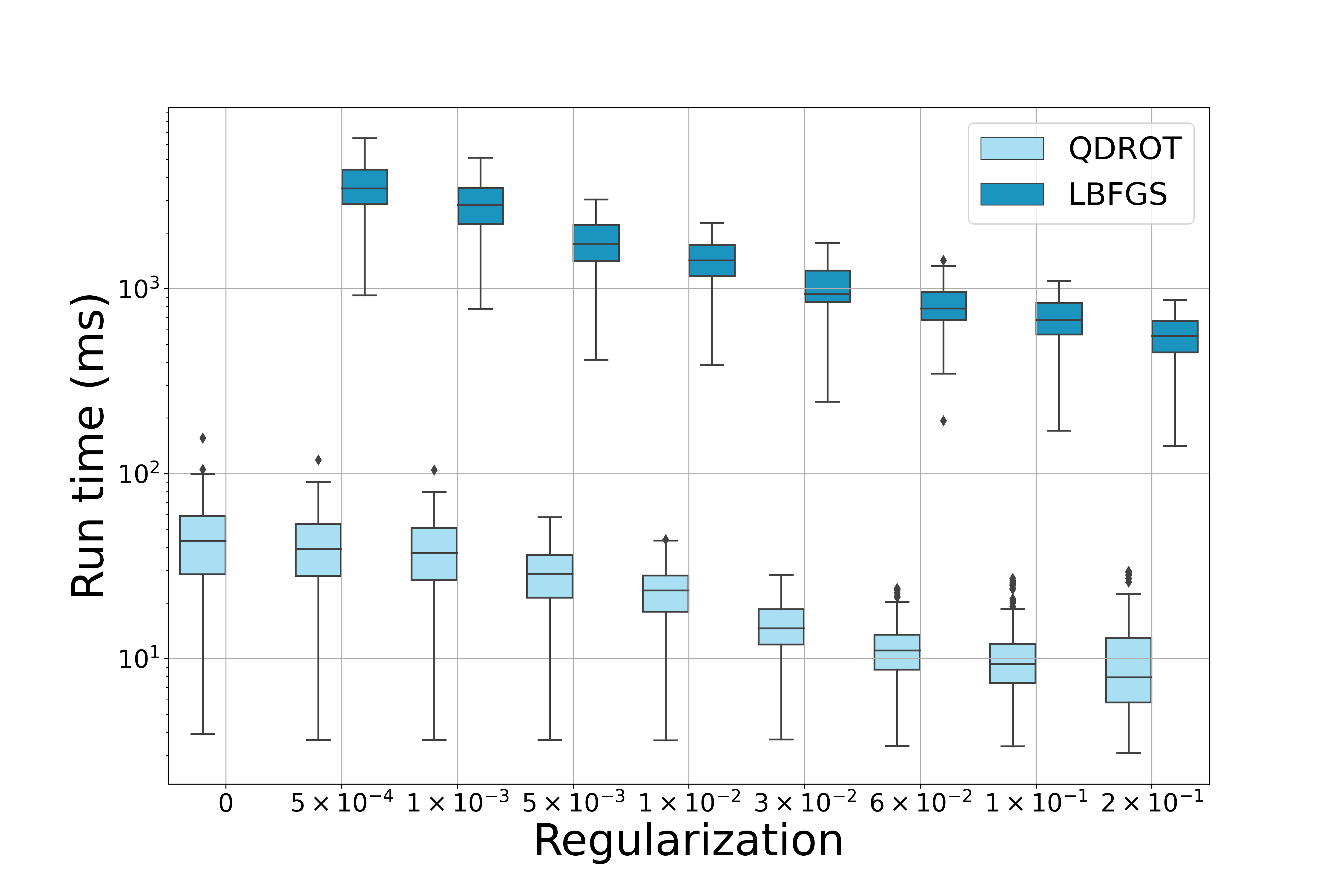}}
        \caption{$m=2000$, $n=3000$}
    \end{subfigure}
    \caption{Comparison between the RDROT (QDROT) and an L-BFGS method applied to the dual for different quadratic regularization parameters. 50 datasets of two different sizes were simulated.}
    \label{fig:quad:ot:benchmark}
\end{figure}

\subsection*{Group Lasso and Domain Adaptation}
OT is the backbone of many domain adaption techniques since it provides a natural framework to estimate correspondences between a training domain and a test domain. These correspondences can subsequently be used to align the training domain with the test domain to improve the test accuracy. Specifically, for a given training data set $\{(\mathbf{x}_s^i, y^i) \}_{i=1}^m$ and a test set $\{\mathbf{x}_t^j \}_{j=1}^n$, where $\{y^i\}_{i=1}^m$ are data labels, we define the cost matrix elements $C_{ij} = d(\mathbf{x}_s^i, \mathbf{x}_t^j)$ using a positive definite function $d$. By computing a transportation plan $X$ associated with \eqref{eq:ot:problem} and cost matrix $C$, one can adapt each data point in the training set into the target domain via $\mathbf{x}_{s\text{ adapt.}}^i = \frac{1}{p_i} \sum_{j=1}^m X_{ij} \mathbf{x}_t^j$. It has been shown that group-lasso regularization can improve the adaptation \cite{courty2016optimal} (cf. the discussion in~\S3). To handle the regularization, it is customary to linearize it and iteratively solve updated OT problems with Sinkhorn  \cite{ courty2014domain, courty2016optimal}. The current implementation group-lasso OT in the Python OT package \texttt{POT} uses this approach \cite{flamary2021pot}. Besides the computational cost of iteratively reconstructing transportation plans, entropic regularization tends to shrink the support of the adapted training data and hence underestimate the true spread in the target domain. This is thoroughly discussed in \cite{feydy2019interpolating}.
We illustrate this effect in Figure~\ref{fig:da:example} where we adapt data sets with different regularizers. 
\begin{figure*}[htb]
    \centering
    \begin{minipage}{0.27\textwidth}
        \centering
        {\includegraphics[width=1.\textwidth]{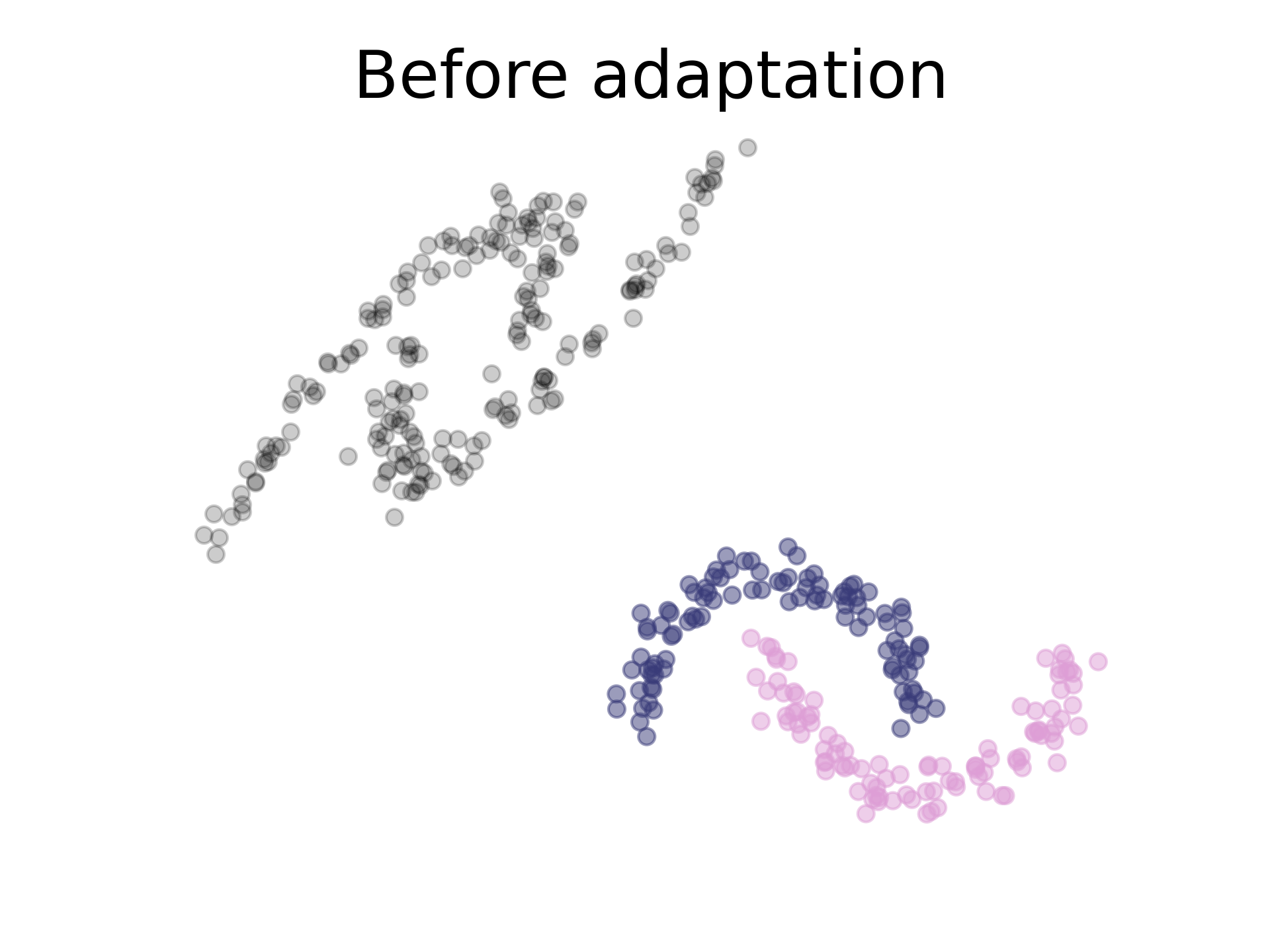}}
    \end{minipage}
    \hskip-0.22in
    \begin{minipage}{0.27\textwidth}
        \centering
        {\includegraphics[width=1.\textwidth]{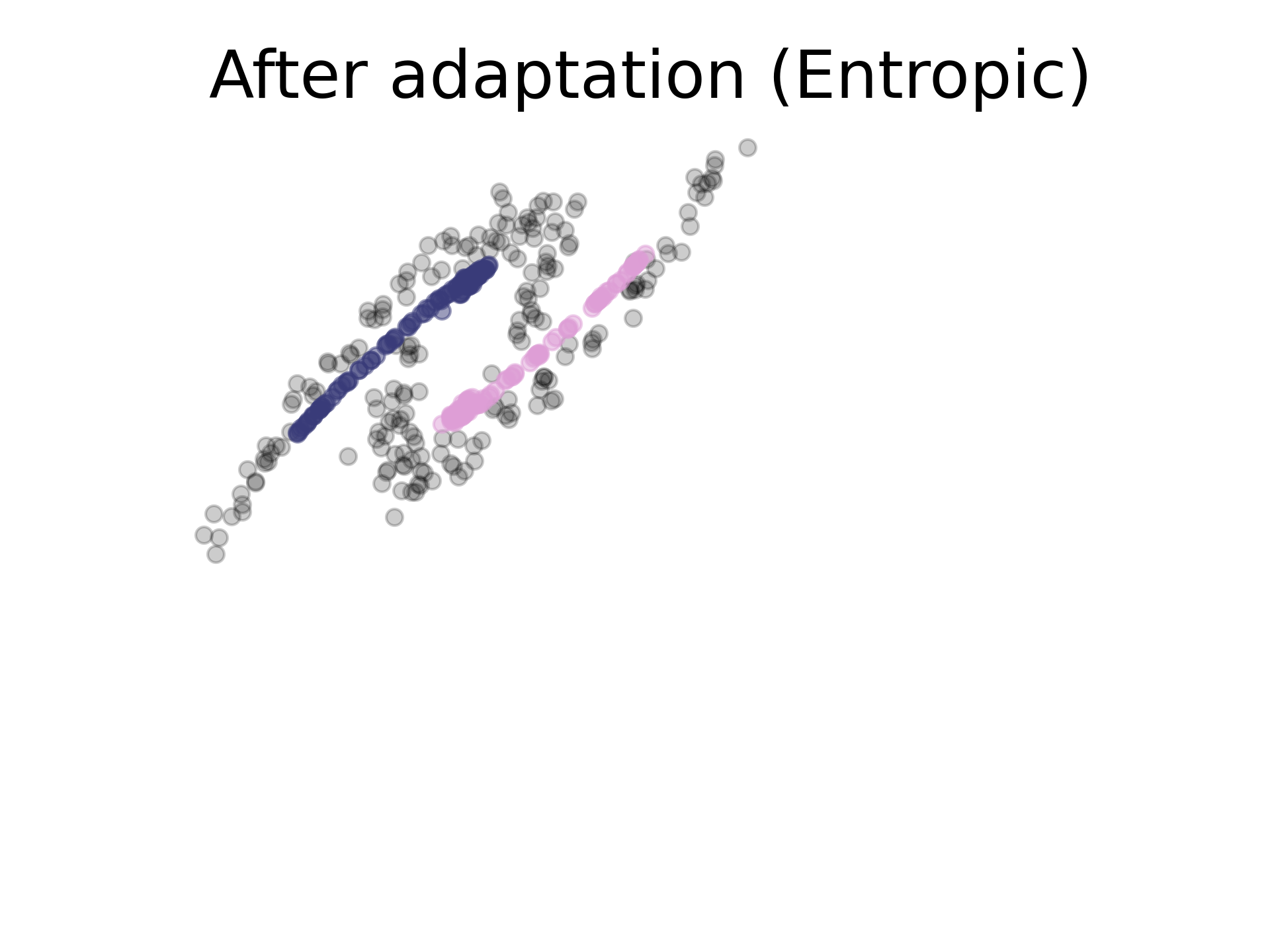}}
    \end{minipage}
    \hskip-0.25in
    \begin{minipage}{0.27\textwidth}
        \centering
        {\includegraphics[width=1.\textwidth]{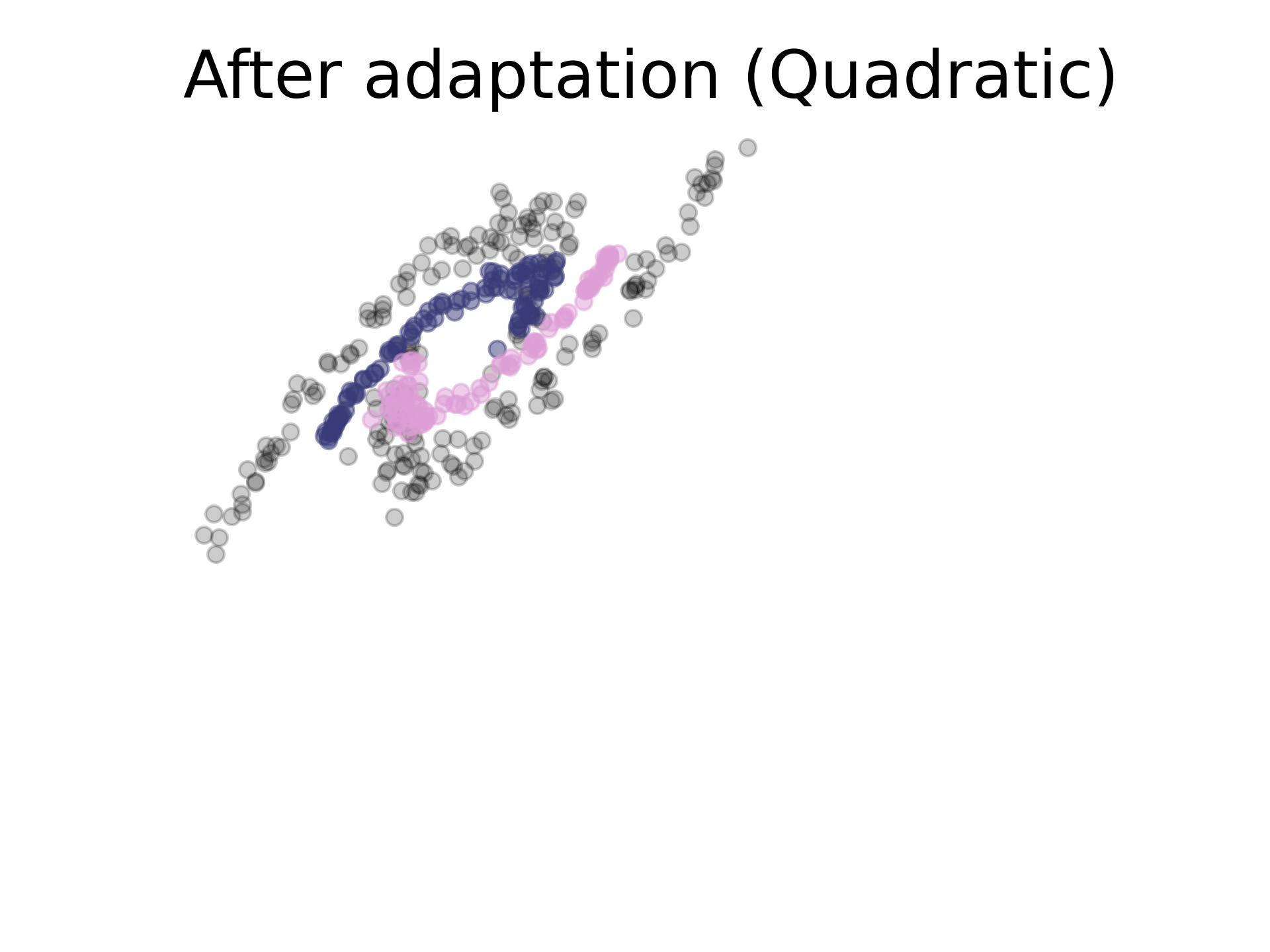}}
    \end{minipage}
    \hskip-0.22in
    \begin{minipage}{0.27\textwidth}
        \centering
        {\includegraphics[width=1.\textwidth]{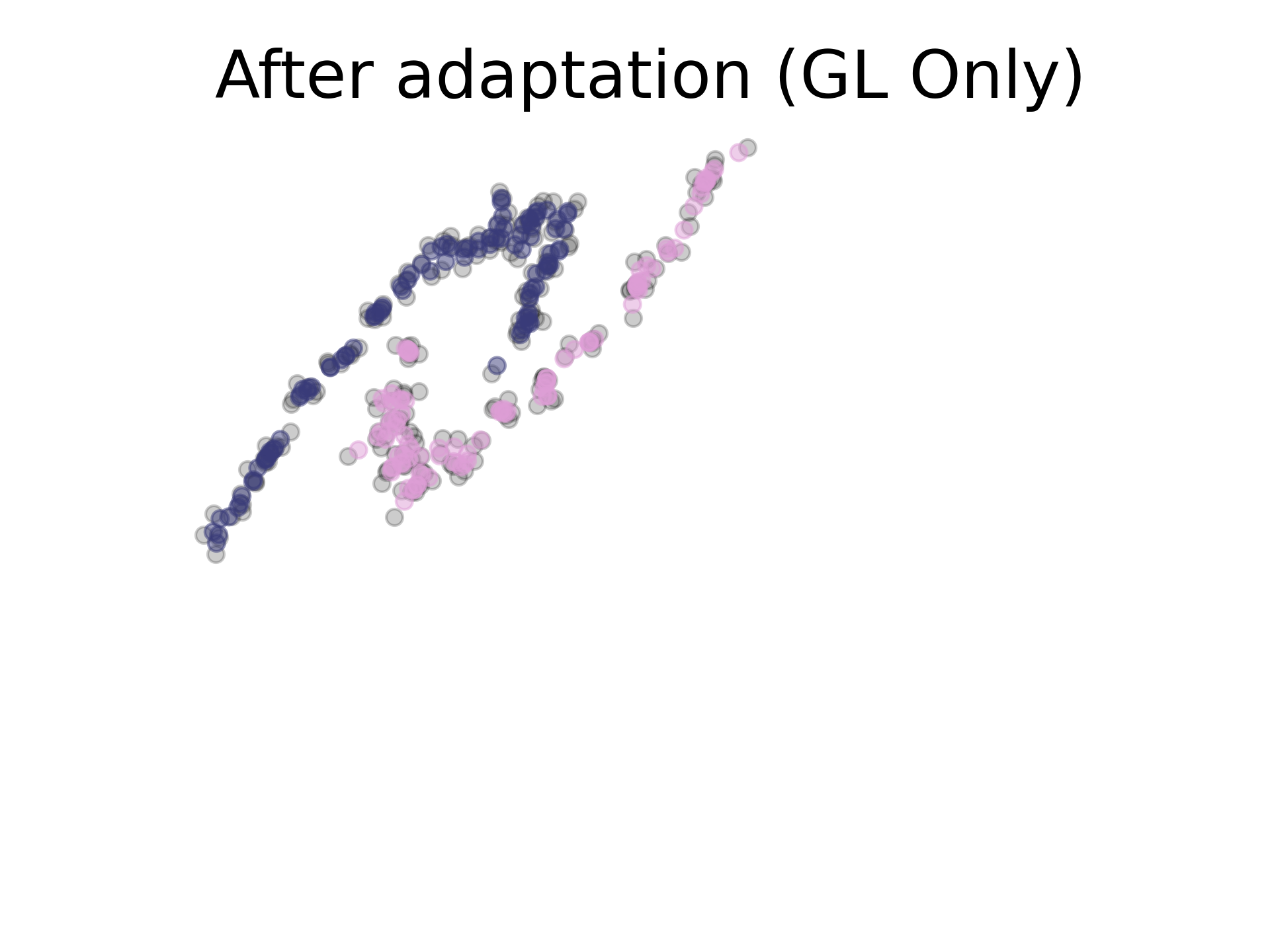}}
    \end{minipage}
      \caption{A toy example illustrating the effect strongly convex regularizations have on the support of the resulting adapted domains. Pink and purple points mark the labeled training set and the gray points the test domain. Notice how the Sinkhorn-based method, as well as the Quadratic Regularized method transfers the source domain samples towards the center of the target domain, while our approach transfers the samples in a way that matches the true variation of the test data }\label{fig:da:example}
\end{figure*}
Of course, the shrinkage can partially be addressed by decreasing the regularization parameter, but this tends to slow down that algorithm considerably and may even lead to numerical instability.
To shed some light on this trade-off, we simulated $50$ datasets with two features, 1500 training samples, and 1000 test samples. The transformed domain of the test set is achieved by a random affine transformation of the source domain. Each set had 2 unique labels that were uniformly distributed among the instances. The labels in the test set were only used for validation. The cost $C_{ij} = \frac{1}{2} \Vert \mathbf{x}_s^i-\mathbf{x}_t^j \Vert^2$ was used, which we normalized so that $\Vert C \Vert_ \infty = 1$. To compare the performance of the algorithms, we computed the time taken to reach a tolerance of $10^{-4}$. To assess the quality of the adaptation, we compute the Wasserstein distance between the adapted samples of each label and the corresponding samples of the test set. This means that the better the alignment, the lower the aggregated distances. For the method using entropic regularization, we varied the regularization parameters $0.001$ to $10$, and the group lasso regularization was set to $0.001$. The results are presented in Table~\ref{tab:da:benchmark}, in which it is clear that our method consistently outperforms the alternative methods, both in terms of adaptation quality, but also time to reach convergence. Additional experiments with other problem sizes and regularization are added to the supplementary material, and the results are consistent with Table~\ref{tab:da:benchmark}.
\begin{table}[htbp]
  \centering
  \caption{Performance benchmark of RDROT for Group-Lasso regularization (GLDROT) against the conjugate-gradient based Sinkhorn algorithm. Median and $10$th and $90$th percentile statistics are included. Our method is  competitive in terms of both speed and adaption quality.}
  \label{tab:da:benchmark}
    \begin{tabular}{l|c c| c c c | c c c}
     \hline
      & \multicolumn{2}{c|}{Reg.}  &\multicolumn{3}{c|}{Runtime (s) $\downarrow$}    & \multicolumn{3}{c}{Agg. $W_2$ dist. $\downarrow$}  \\
      Method & Ent & GL & Median  & $q10$ & $q90$ & Median  & $q10$ & $q90$ \\
    \hline
    GLSK \cite{courty2017joint} & \texttt{1e-3} & \texttt{1e-3} &3.77 & 3.68 & 8.90 & 0.311 & 0.0657 & 6.48 \\
    GLSK \cite{courty2017joint} & \texttt{1e-1}   &\texttt{1e-3} & 3.73 & 3.71 & 3.77 & 8.24 & 3.47 & 31.6 \\
    GLSK \cite{courty2017joint} & \texttt{1e+2}  & \texttt{1e-3}&1.86 & 1.86 & 1.88 & 52.8 & 10.9 & 311 \\ \hline 
    GLDROT (ours) & - & \texttt{1e-3} & 0.0619 & 0.0475 & 0.0771 & \textbf{0.0576} & 0.0262 & 0.182 \\
    GLDROT (ours) & - & \texttt{5e-3}& \textbf{0.0384} & 0.0306 & 0.0474 & 0.0879 & 0.0358 & 0.274 \\
    \hline
    \end{tabular}%
  \label{tab:addlabel}%
\end{table}%

\subsection*{Training of generative models}
As an illustration of the efficiency of GPU version of the algorithm, we use our PyTorch wrapper to implement the Minibatch Energy Distance~\cite{salimans2018improving}, which is used as the loss function of a Generative Adversarial Network. We performed experiments of image generation on MNIST and CIFAR10 datasets, and some of the generated samples are shown in Figure \ref{fig:gan:example}. The details of the experiments are included in the supplementary material. The resulting generator gives sharp images, and the training time was better or comparable with alternative frameworks, despite the low amount of regularization used.

\begin{figure*}[htb]
    \centering
    \begin{minipage}{0.4\textwidth}
        \centering
        {\includegraphics[trim=0cm 11.4cm 9.12cm 0cm,clip,width=1.\textwidth]{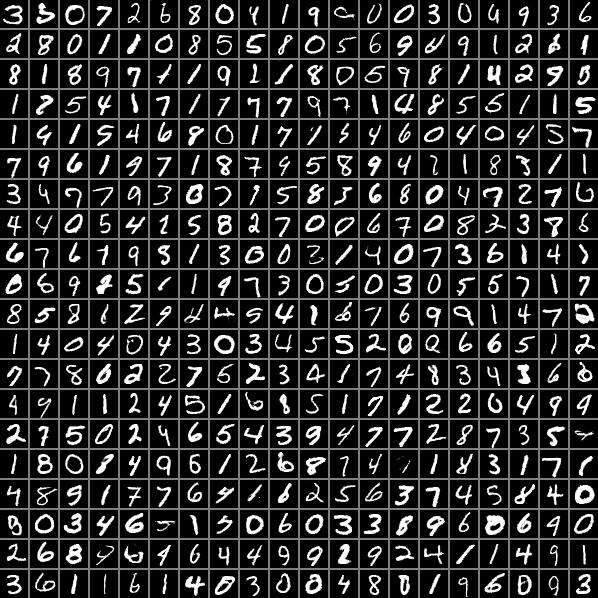}}
    \end{minipage}
    \hskip 0.42in
    \begin{minipage}{0.4\textwidth}
        \centering
        {\includegraphics[trim=2.55cm 6.00cm 7.724cm 6.86cm,clip,width=1.\textwidth]{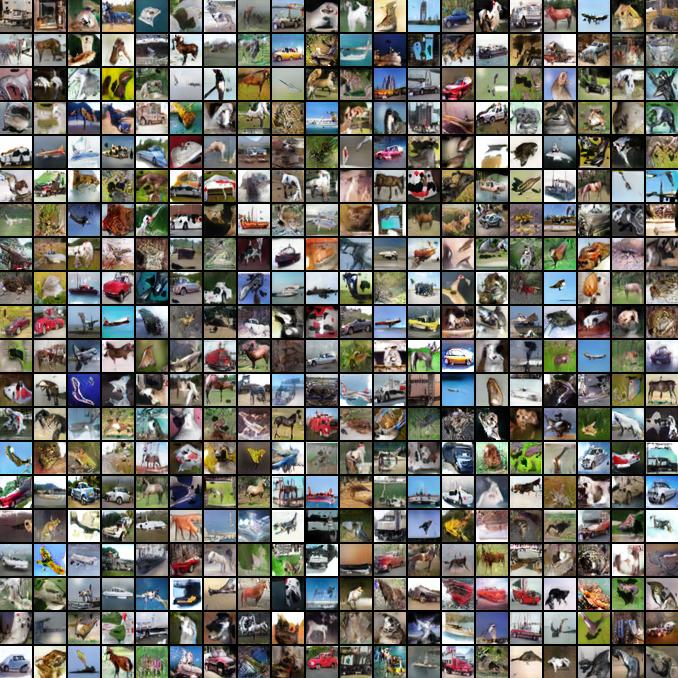}}
    \end{minipage}
    
      \caption{Generated samples by GANs trained on MNIST and CIFAR10 datasets with the Minibatch Energy Distance. The original sizes of samples are (28,28) and (32,32), respectively.}\label{fig:gan:example}
\end{figure*}

\section{Conclusions}

The ability to solve large regularized optimal transport problems quickly has the potential to improve the performance of a wide range of applications, including domain adaption, learning  of generative models, and more. In this paper, we have shown that the Douglas-Rachford splitting technique can handle a broad class of regularizers in an integrated and computationally efficient manner. 
Using a carefully chosen problem split, we derived an iterative algorithm that converges rapidly and reliably, and whose  operations are readily parallelizable on GPUs.
The performance of our algorithm is supported by both theoretical arguments and extensive empirical evaluations. In particular, the experiments demonstrate that the resulting algorithms can be up to two orders of magnitude faster than the current state-of-the-art. We believe that this line of research has the potential to make regularized OT numerically tractable for a range of tasks beyond the ones discussed in this paper.

\newpage
\bibliographystyle{plain}
\bibliography{references}



\newpage
\appendix

\section{Theoretical results}
\subsection{Proof of Lemma 3.1}
To prove Lemma 3.1, the following result is helpful:
\begin{lemma}\label{lemma:sparsity:preserving}
    Given that $h$ is sparsity promoting and $\stepsize>0$, then $X_{ij} = 0$ implies that $(\prox{\rho h}{X})_{ij} = 0$. Moreover, if $X_{ij} > 0$, then $(\prox{\rho h}{X})_{ij} \geq 0$.
\end{lemma}
\begin{proof}
    Let $Y = \prox{\rho h}{X}$. 
    By the definition of the proximal operator, every $Z \in \R^{m \times n}$ satisfies
    $$h(Y) + \frac{1}{2\rho} \Vert Y-X \Vert_F^2 \leq h(Z) + \frac{1}{2\rho} \Vert Z-X \Vert_F^2$$
    or 
    $$h(Y) + \frac{1}{2\rho} \bigg(\Vert Y-X \Vert_F^2 - \Vert Z-X \Vert_F^2 \bigg) \leq h(Z). $$
    Moreover, since $h$ is sparsity promoting $h(Z)\leq h(Y)$, and it holds that
    \begin{align}\label{eq:prox:ineq}
        \frac{1}{2\rho} \bigg(\Vert Y-X \Vert_F^2 - \Vert Z-X \Vert_F^2 \bigg) \leq 0.
    \end{align}
 Let $Z$ be equal to $Y$ in all entries but the $(i,j)$th, which is set to zero.
If $X_{ij}=0$, then 
    $$\frac{1}{2\rho}\bigg(\Vert Y-X \Vert_F^2 -  \Vert Z-X \Vert_F^2 \bigg) =\frac{1}{2\rho} Y_{ij}^2$$ 
     and invoking \eqref{eq:prox:ineq} gives that $Y_{ij}=0$. If $X_{ij}> 0$, on the other hand, then

     $$\frac{1}{2\rho}\bigg(\Vert Y-X \Vert_F^2 -  \Vert Z-X \Vert_F^2 \bigg) =\frac{1}{2\rho}\bigg((Y_{ij}-X_{ij})^2 - X_{ij}^2 \bigg),$$
    and \eqref{eq:prox:ineq} gives
    $$ (Y_{ij}-X_{ij})^2 \leq X_{ij}^2$$
    which implies that $Y_{ij}\geq 0$. The proof is complete.
\end{proof}

Lemma~\ref{lemma:sparsity:preserving} and Theorem~1 in \cite{yu2013decomposing} allows to prove the composition lemma quite easily. We end this section with the proof.

\begin{proof}{\textbf{(Lemma 3.1)}}
    For $p(X) = \iota_{\R_+^{m \times n}}(X)$, 
   it holds that 
    \begin{align*}
     \prox{\rho p}{X} &= [X]_+
     \intertext{and}
     \partial p(X) &= N_{\R_+^{m \times n}}(X) \supseteq \{0\}
    \end{align*}
for all $X\geq 0$, where 
$N_{\R_+^{m \times n}}(X) = \{Y \in \R^{m\times n}: Y_{ij}=0 \text{ if } X_{ij}>0,\, Y_{ij}\leq 0 \text{ if } X_{ij}=0\}.$

By Lemma~\ref{lemma:sparsity:preserving}, if $X_{ij}=0$, then $(\prox{\stepsize h}{X})_{ij}=0$, and hence $(\partial p(X))_{ij} = (\partial p(\prox{\stepsize h}{X}))_{ij}$. Further, for positive entries $X_{ij}>0$, $(\partial p(X))_{ij} = \{0\}$, which implies that $(\partial p(\prox{\stepsize h}{X}))_{ij} \supseteq (\partial p(X))_{ij}$. Altogether, this means that 
$$ \partial p(\prox{\stepsize h}{X}) \supseteq \partial p(X).$$
Theorem 1 in \cite{yu2013decomposing} then gives
$$ \prox{\rho p + \rho h}{X} = \prox{\rho h}{\prox{\rho p}{X}} = \prox{\rho h}{[X]_+}$$
and
\begin{align*}
\prox{\rho f }{X} &= \prox{\rho \InP{C}{\cdot} +\rho p + \rho h}{X} \\
&= \prox{\rho p + \rho h}{X-\rho C} \\
&= \prox{\stepsize h}{[X-\stepsize C]_+}.
\end{align*}
\end{proof}
\subsection{On the non-degeneracy condition of Asssumption~1}
Assumption~1 is violated if $Y^\ast$ is a fixed point to equation (3) and $X^\star = \prox{\rho f}{Y^\star}$, but the regularity condition not fulfilled. That is, if
\begin{align*}
    \frac{1}{\stepsize}\bigg( Y^\ast - X^\ast \bigg) \in  \partial f(X^\star), \text{ while   }
    \frac{1}{\stepsize}\bigg( Y^\ast - X^\ast \bigg) \notin  \ri \partial f(X^\star). \nonumber
\end{align*}
We can write these conditions more conveniently as
\begin{align}
    \frac{1}{\stepsize}\bigg( Y^\ast - X^\ast \bigg) &\in  \partial f(X^\star) \setminus \ri \partial f(X^\star). \label{eqn:reg_conditions}
\end{align}

Let us now explore to what degree Assumption 1 can be seen as restrictive when it is invoked in Theorem~1. To this end, we also assume that $h$ is twice continuously differentiable for $X>0$.

Recall that $f = \InP{C}{X} + \iota_{\R_+^{m \times n}} + h(X)$ and note that the inclusion (\ref{eqn:reg_conditions}) cannot be satisfied if $X^\ast>0$, since $\partial f(X^\star)$ is a singleton. We therefore focus on the case when $\partial f$ is set-valued, i.e. when $X_{ij}^\ast=0$ for some $(i,j)$. Since $h(X)$ is differentiable for $X>0$, we only need to consider the zero entries. If $X_{ij}^\ast=0$, $\partial \iota_{\R_+^{m \times n}} = N_{\R_+^{m\times n}}$ implies that  $( \partial \iota_{\R_+^{m \times n}} \setminus \ri \partial \iota_{\R_+^{m \times n}})_{ij} = 0$ and thus that $ \partial \iota_{\R_+^{m \times n}} \setminus \ri \partial \iota_{\R_+^{m \times n}} = 0$. Consequentially, $\partial f(X^\star) \setminus \ri \partial f(X^\star) = C + \partial h(X^\ast) \setminus \ri \partial h(X^\ast)$, and 
\begin{align*}
    \frac{1}{\stepsize}\bigg( Y^\ast - X^\ast \bigg) = C + g, \text{ where } g \in \partial h(X^\ast) \setminus \ri \partial h(X^\ast).
\end{align*}
By the convergence of the algorithm, equation (5) in the main document implies that there exist two vectors $\phi^\star$ and $\varphi^\star$ such that $Y^\star - X^\star =\phi^\star \ones_n^\top + \ones_m {\varphi^\star}^\top$. This gives that the cost is on the form
$$ C = \rho^{-1}(\phi^\star \ones_n^\top + \ones_m {\varphi^\star}^\top) - g.$$
In other words, it is only cost matrices with a specific (and rather particular) structure that violate Assumption~1. For instance, if $h$ is smooth, e.g. $h=\lambda \Vert \cdot \Vert_F^2$, Assumption~1 is only violated when the cost can be written as a sum of two given rank-1 matrices. This is rarely met in real applications. 

\subsection{Proof of Theorem 1 and 2}
To prove the theorem we need the notion of partial smoothness: \\
\begin{definition}[Partial Smoothness] A lower semi-continuous function $f: \R^n \to \R$ is said to be partly smooth at $x^\star \in \R^n$ with respect to the $C^2$-manifold $\mc{M} \subset \R^n$ if: 
\begin{enumerate}[(i)]
    \item The restriction of $f$ to $\mc{M}$: $f |_\mc{M}$ is $C^2$ around $x^\star$,
    \item $\mc{T}_\mc{M}(x^\star) = (\mathrm{par}\, \partial f(x^\star) )^\perp $, where $(\mathrm{par}\, \partial f(x^\star) )^\perp$ denotes the orthogonal complement of the smallest subspace parallel to $\partial f(x^\star)$,
    \item $\partial f(x)$ is continuous at $x^\star$ relative to $\mc{M}$. 
\end{enumerate}
\end{definition}

Let $X^\star$ be a solution to the OT problem that fulfills $X_k \to X^\star$, and let
\begin{align}
    \mc{M}_1 = \{X \in \R_+^{m\times n}: X_{ij}=0, \text{ if } X_{ij}^\star=0,\, X_{ij}>0 \text{  if }X_{ij}^\star>0 \}, \quad
    \mc{M}_2 = \mc{X}.
\end{align}
Recall that $f(X) = \InP{C}{X} + \iota_{\R_+^{m \times n}}(X) + h(X)$, and $g(X) =  \iota_{\mathcal{X}}(X)$. Since $h$ is assumed to be twice continuously differentiable over $\R_{++}^{m\times n}$, by the conditions in the theorem, $f$ is partly smooth with respect to $\mc{M}_1$. Further, since $g$ is constant over $\mc{M}_2$, it is also partly smooth over $\mc{M}_2$. The optimality conditions of the regularized OT problem are
\begin{align}\label{eq:rewritten:opt:cond}
\begin{aligned}
    0 &\in  \partial f(X^\star) - \frac{1}{\stepsize}\bigg( Y^\ast - X^\ast \bigg) \\
    0 &\in  \partial g(X^\star)+  \frac{1}{\stepsize}\bigg( Y^\ast - X^\ast \bigg).
\end{aligned}
\end{align}
Moreover, since $\partial g = N_{\mathcal{X}}$, and $\mc{X}$ is an affine subspace, $\partial g = \ri \partial g$. Therefore, Assumpion~1 extends to the following the non-degeneracy condition:
\begin{align}\label{eq:relin:fg}
\begin{aligned}
    0 &\in  \ri \partial f(X^\star) - \frac{1}{\stepsize}\bigg( Y^\ast - X^\ast \bigg) \\
    0 &\in  \ri \partial g(X^\star)+  \frac{1}{\stepsize}\bigg( Y^\ast - X^\ast \bigg).
\end{aligned}
\end{align}
With the partial smoothness and the non-degeneracy condition established, we have all the  building blocks needed for proving Theorem 1 and 2.

\begin{proof}{(\textbf{Theorem 1 and 2})}
    By invoking the first order optimality condition associated with the first update in the DR-splitting scheme in (4), we get:
 \begin{align*}
    X_{k+1} = \prox{\stepsize f}{Y_k} \implies 0 \in \frac{1}{\rho}\bigg(X_{k+1}-Y_k\bigg) + \partial f(X_{k+1})
 \end{align*}
 or
  \begin{align}\label{eq:x:inclusion}
    \frac{1}{\rho}\bigg(Y_k-X_{k+1} \bigg) \in \partial f(X_{k+1}).
 \end{align}
  Similarly, the second update in (4) is associated with the inclusion:
\begin{align*}
    Z_{k+1} = \prox{\stepsize g}{2 X_{k+1} - Y_k} \implies 0 \in \frac{1}{\stepsize} \bigg(  Z_{k+1}-2X_{k+1} + Y_k \bigg) + \partial g(Z_{k+1})
 \end{align*}
 or equivalently
 \begin{align}\label{eq:z:inclusion}
    \frac{1}{\stepsize} \bigg(Y_k - X_{k+1} \bigg) - \frac{1}{\stepsize}\bigg(  Y_{k+1}-Y_{k} \bigg) \in g(Z_{k+1}).
 \end{align}
 where we used the third equation of the DR-update, i.e $Y_{k+1}-Y_k = Z_{k+1}-X_{k+1}$.
%
Notice how \eqref{eq:x:inclusion} and \eqref{eq:z:inclusion} relate closely to the optimality conditions  \eqref{eq:rewritten:opt:cond}.
Now consider the function $F(X, Z) = f(X) + g(Z) - \stepsize^{-1}\InP{X-Z}{Y^\ast - X^\ast}$. Since 
$X_{k+1} \in \R_+^{m \times n}$ and $Z_{k+1} \in \mc{X}$ for all $k \geq 1$, it holds that 
$F(X_{k+1}, Z_{k+1}) = \InP{C}{X_{k+1}} + h(X_{k+1}) - \stepsize^{-1}\InP{X_{k+1}-Z_{k+1}}{Y^\ast - X^\ast}$. Since $f+g$ are convex, closed, and attains a minimizer, Theorem~25.6 in \cite{bauschke2011convex} can be used to establish that $Y_k \to Y^\star$, when $k \to \infty$, where $Y^\star$ is a fixed point of the DR-update. In particular, 
$$\Vert Y_{k+1} -Y_k\Vert = \Vert (Y_{k+1} - Y^\star) -(Y_k - Y^\star )\Vert \leq \Vert Y_{k+1} - Y^\star \Vert + \Vert Y_{k} - Y^\star \Vert \to 0,$$ 
which can be used to deduce that the third term of $F$ tends to zero as $k\to \infty$: 
$$\vert \InP{X_{k+1}-Z_{k+1}}{Y^\ast - X^\ast}\vert \leq  \Vert Z_{k+1} - X_{k+1} \Vert \Vert Y^\ast - X^\ast \Vert = \Vert Y_{k+1} - Y_{k} \Vert \Vert Y^\ast - X^\ast \Vert \to 0.$$
Consequentially,
\begin{align}\label{eq:F:limit}
    F(X_{k+1}, Z_{k+1}) \to \InP{C}{X^\star} + h(X^\star) = f(X^\star) + g(X^\star).
\end{align}
As $h$ is assumed to be convex, proper, and lower semicontinuous, we have that $h$ is subdifferentially continuous in its domain (including at $X^\star$) \cite{rockafellar2009variational}. This gives
\begin{align*}
    \mathrm{dist}\bigg\{0,\, \partial_X F(X_{k+1},Z_{k+1})\bigg\} &= 
    \mathrm{dist}\bigg\{0,\, \partial f(X_{k+1}) - \frac{1}{\stepsize}\bigg( Y^\ast - X^\ast \bigg) \bigg\}\\ &\leq \stepsize^{-1} \Vert (Y_k - Y^\ast)- (X_{k+1}-X^\ast)  \Vert \\
    &\leq \stepsize^{-1}\Vert Y_k - Y^\ast \Vert + \stepsize^{-1}\Vert  X_{k+1}-X^\ast \Vert \\
    &\leq 2\stepsize^{-1}\Vert Y_k - Y^\ast \Vert \\
    &\to 0.
\end{align*}
Here, the first inequality follows from that $\mathrm{dist}\{x, S \} = \inf\{ \Vert x-y \Vert: y \in S\} \leq \Vert x-y \Vert,$ for all $y\in S$. The second is the Cauchy Schwarz inequality, and the third inequality follows from the non-expansiveness of proximal operators: 
$$\Vert  X_{k+1}-X^\ast \Vert = \Vert  \prox{\rho f}{Y_{k}}-\prox{\rho f}{Y^\ast} \Vert \leq \Vert Y_k - Y^\ast \Vert.$$ Similarly,
\begin{align*}
    \mathrm{dist}\bigg\{0,\, \partial_Z F(X_{k+1},Z_{k+1})\bigg\} &=
    \mathrm{dist}\bigg\{0,\, \partial g(X_{k+1}) + \frac{1}{\stepsize}\bigg( Y^\ast - X^\ast \bigg) \bigg\}  \\ &\leq \stepsize^{-1} \Vert (X_{k+1}-X^\ast) - (Y_k - Y^\ast)  - (Y_{k+1}-Y_{k}) \Vert \\
    &\leq 2\stepsize^{-1}\Vert Y_k - Y^\ast \Vert + \stepsize^{-1} \Vert Y_{k+1}-Y_{k} \Vert\\
    &\to 0.
\end{align*}
Hence, $\text{dist}\bigg\{0,\, \partial F(X_{k+1},Z_{k+1})\bigg\} \to 0$. This together with \eqref{eq:relin:fg} and \eqref{eq:F:limit}, and the fact that $f$ is partially smooth at $X^\star$ with respect to $\mc{M}_1$, we apply Theorem 5.3 in \cite{hare2004identifying} which proves the theorem. To prove Theorem 2, one can directly invoke Theorem 5.6 in \cite{liang2017local}.     
\end{proof}
\section{Initialization}
When using the initialization $X_0 = pq^\top$, $\phi_0 = \mathbf{0}_m$, and $\varphi_0 = \mathbf{0}_n$, and the default stepsize, the updates in (6) gives that 
\begin{align*}
    X_k &= \mathbf{0}_{m \times n} \\
    \phi_k &= (k+1)n^{-1}(p + (m+n)^{-1}) \\
    \varphi_k &=(k+1)m^{-1}(q + (m+n)^{-1})
\end{align*}
when $k = 1,\,2,\dots N$, where 
$$N=\min_{ij} \lceil C_{ij} mn (m+n)^{-1}(m p_i + n q_j + 1)^{-1} - 1 \rceil = O(\max(m,n)).$$ 
When the cost is normalized, we can use the rough approximation, $C_{ij} \sim 1$, $p_i \sim m^{-1}$, and $q_j \sim n^{-1}$, to get the simple initialization strategy: $X_0 = \mathbf{0}_{m \times n}$, $\phi_0 = (3(m+n))^{-1}(1 + m/(m+n)) \ones_m$, and $\varphi_0 = (3(m+n))^{-1}(1 + n/(m+n)) \ones_n$, that skips these first $N$ iterations.
\section{Derivation of Dual and Stopping criterion}\label{appendix:dual:stopping:criterion} 
A simple calculation reveals that $g(X)= \iota_{\mathcal{X}}(X)$ has the Fenchel conjugate
\begin{align*}
    g^*(U) = \begin{cases}
    \langle{U},{pq^\top}\rangle, & \text{if } U = \mu \ones_n^\top + \ones_m \nu^\top \\
    \infty, & \text{otherwise.}
    \end{cases}
\end{align*}
Hence, any feasible dual variable is on the form $U = \mu \ones_n^\top + \ones_m \nu^\top$, for which $g^*(U) = p^\top \mu + q^\top \nu$. Further, the conjugate of $f(X) =  \InP{C}{X} + \iota_{\R_+^{m \times n}}(X) + h(X)$ can be expressed as
\begin{align*}
    f^*(U) &= \sup_X \InP{U}{X} - f(X) \\
           &= \sup_{X\geq 0} \InP{U-C}{X} - h(X) \\
           &\leq  \sup_{X\geq 0} \InP{[U-C]_+}{X} - h(X) \\
           &\leq \sup_{X} \InP{[U-C]_+}{X} - h(X) \\
           &= h^*([U-C]_+).
\end{align*} 
Moreover, since $h$ is sparsity promoting,
\begin{align*}
    h^*([U-C]_+) &=  \sup_{X} \InP{[U-C]_+}{X} - h(X) \\
                 &\leq  \sup_{X} \InP{[U-C]_+}{[X]_+} - h([X]_+)\\
                 &=\sup_{X\geq 0} \InP{[U-C]_+}{X} - h(X) \\
                 &= f^*(U), 
\end{align*}
meaning that $f^*(U) = h^*([U-C]_+)$. This gives the dual problem
\begin{align}\label{eq:derived:dual}
    \maximize_{U = \mu \ones_n^\top + \ones_m \nu^\top} \,\,\,  -f^*(-U) - g^*(U) = -p^\top \mu - q^\top \nu - h^*([-\mu \ones_n^\top - \ones_m \nu^\top-C]_+).
\end{align}
To relate the iterates to the optimality condition, we define $X_{k+1} = \prox{\stepsize f}{Y_k}$, $Z_{k+1}=\prox{\stepsize g}{2 X_{k+1} - Y_k}$, and $U_{k} := (Y_{k-1}-X_k)/\stepsize$ and observe that
\begin{align*}
    \stepsize^{-1}(Z_{k+1}-X_{k+1}) = \stepsize^{-1}(X_{k+1} - Y_k - \phi_k \ones_n - \ones_m \varphi_k^\top) = -U_{k+1} -  (\stepsize^{-1}\phi_k) \ones_n - \ones_m (\stepsize^{-1}\varphi_k)^\top.
\end{align*}
By theorem 25.6 in \cite{bauschke2011convex}, $Z_{k+1}-X_{k+1} \to 0$, and $U_{k+1} \to U^\star$ where $U^\star$ is a solution to \eqref{eq:derived:dual}. Hence $ (\stepsize^{-1}\phi_k) \ones_n + \ones_m (\stepsize^{-1}\varphi_k)^\top \to U^*$, meaning that $\stepsize^{-1}\phi_k$ and $\stepsize^{-1}\varphi_k$ will correspond to $\mu$ and $\nu$ in \eqref{eq:derived:dual} respectively. Using that $r_k$ and $s_k$ in (6) are directly related to the primal residual, we can derive the following stopping criteria:
\begin{align}\label{eq:stopping:criteria}
    r_\mathrm{primal} : =& \max ( \Vert X_k\ones_n-p \Vert, \Vert X_k^\top \ones_m-q \Vert) = \max (\Vert r_k \Vert, \Vert s_k \Vert ), \nonumber\\
    \mathrm{gap} :=&  \InP{C}{X_k} + h(X_k) - \stepsize^{-1}(p^\top \phi_k + q^\top \varphi_k) + h^\ast(\stepsize^{-1}([\phi_k \ones_n^\top + \ones_m \varphi_k^\top - \stepsize C]_+) \nonumber\\
    =& \InP{C}{X_k} - \stepsize^{-1}(p^\top \phi_k + q^\top \varphi_k) + \Delta_h(X_k, \, \phi_k, \varphi_k), \nonumber
\end{align}
where $\Delta_h$ denotes the duality gap deviation compared to the unregularized problem. Further, notice that by letting $\bar{X}_{k+1} = [X_k + \phi_k \ones_n^\top + \ones_m \varphi_k^\top - \stepsize C]_+$, we can express $[\phi_k \ones_n^\top + \ones_m \varphi_k^\top - \stepsize C]_+ = [\bar{X}_{k+1} -X_{k}]_+$, which facilitates evaluating the duality gap. Notice that $h^\ast$ may not be finite. In such settings, the dual is maximized over its effective domain, resulting in the addition of dual constraints. To account for these, one must include a dual residual that measures the constraint violation. We include some examples of stopping criteria if the following table.

\begin{table*}[!ht]
    \centering
    \caption{Termination criteria for a selection of regularizers. Here we let $\Vert \cdot \Vert$ denote an arbitrary norm that is sparsity promoting.}
    \vskip 0.15in
    \begin{center}
    \begin{small}
    \begin{sc}
    \label{tab:term:crit}
    \begin{tabular}{c c c}
        \toprule
         $h$ & $\Delta$ & $r_\mathrm{dual}$ \\ \midrule
         $0$ &  $0$ & $\Vert {X}_{k+1} -X_{k}]_+ \Vert_F$ \\
         $\Vert X \Vert_F^2$ & $ \Vert [(1+\rho){X}_{k+1} -X_{k}]_+ \Vert_F^2$ & $0$ \\
         $ \Vert X \Vert$ & $0$ & $\Vert \prox{h}{[\bar{X}_{k+1} -X_{k}]_+} \Vert_F$ \\
         \bottomrule
    \end{tabular}
    \end{sc}
    \end{small}
    \end{center}
    \vskip -0.1in
    \end{table*}
    
\section{Backpropagation through regularized OT-costs}
We let $h_c = h + \iota_{\R_+^{m \times n}} + \iota_{\mc{X}}$, and define the OT-cost as follows: $\mathrm{OT}_h(C) = \inf_X \InP{C}{X} + h_c(X)$. Note that OT cost is equal to the optimal value of the OT problem (1), i.e. 
if $X^\star$ is a solution to (1), then $\mathrm{OT}_h(C) = \InP{C}{X^\star} + h_c(X^\star)$. The OT cost is also closely related to the Fenchel conjugate of $h_c$, as $\mathrm{OT}_h(C) = - \bar{h}^\ast(-C)$, and hence $\bar{h}^\ast(-C) + \bar{h}(X^\star) =  \InP{-C}{X^\star}$. Invoking subdifferential properties of the Fenchel conjugate gives that $X^\star \in \partial \bar{h}^\star(-C)$. Further
$$\partial_C( - \mathrm{OT}_h(C)) = \partial_C h_c^\ast(-C) \ni -X^\star.$$
In particular, whenever the solution is unique, such as when $h$ is strongly convex, then $\nabla_C \mathrm{OT}_h(C) = X^\star$. Therefore, using the transportation plan estimated via RDROT is a reasonable gradient estimate of the OT cost.
\section{GPU Implementation}

\subsection{GPU Architecture Basics}

Before introducing the \replaced{detailed design}{design details} of the GPU kernels, we recall some basic properties of \deleted{the} modern GPU \replaced{architectures}{architecture}. Since the kernel design is based on CUDA\added{,} which is compatible with NVIDIA GPUs, we \deleted{only} focus on the architecture of NVIDIA GPUs\replaced{.}{ here.}

\subsubsection{Thread Organization}

\replaced{The}{A Streaming Multiprocessor (SM) is a} fundamental unit of computation in NVIDIA GPUs\added{ is the Streaming Multiprocessor (SM)}. It is responsible for executing \added{the} parallel computations and managing \added{the} resources on the GPU. Each SM consists of multiple CUDA cores, \added{a} shared memory, registers, and other components that enable efficient parallel processing.

The parallel computations are organized in threads, grouped according to the following hierarchy:
\deleted{As for the organization of the parallel threads, it can be introduced from the lowest level to the highest level as follows.}
\begin{enumerate}
    \item The \emph{thread} is the basic unit of code. All threads execute the same code, but they are endowed with an ID that can be used to parameterize memory access and control decisions.
    \item \deleted{Block:} A \emph{block} is a group of threads that execute together on an SM in the GPU. Threads within a block can cooperate and communicate using shared memory.
    \item \deleted{Grid:} The \emph{grid} refers to the entire set of blocks that will be executed on the GPU. It represents the overall organization of parallel computation.
\end{enumerate}

The smallest unit of thread execution on an NVIDIA GPU is the \emph{warp}. It consists of 32 threads that are scheduled and executed together on a single SM. All threads within a warp execute the same instruction at the same time (but on different data, parameterized by the thread and block IDs). The reduction within the warp is highly optimized and efficient.

\subsubsection{Memory Hierarchy}

\replaced{NVIDIA GPUs, such as the NVIDIA Tesla V100, have}{
As for the memory hierarchy of NVIDIA GPUs, take NVIDIA Tesla V100 GPU as an example, there are} more than eight types of memory including global memory, shared memory, registers, texture memory, constant memory, etc. We only focus on the first three \replaced{memory types, since they are the most}{types of memory which are highly} relevant to the kernel design in the next section.

\begin{enumerate}
    \item On NVIDIA GPUs, the \emph{global memory} \deleted{Global memory: The \emph{global memory} of NVIDIA GPUs} refers to the main memory available on the GPU. It is a high-speed memory space used for storing data and instructions that are accessible by all the CUDA cores in the GPU. The global memory is the entry for exchanging data with the system RAM (random access memory). The global memory lies in the highest level of the memory hierarchy and has the largest capacity and the slowest \replaced{access}{accessing} speed\deleted{ among all types of memory}. NVIDIA Tesla V100 has 32GB \replaced{of}{as the} global memory. An important feature of global memory is coalesced memory access, which will be discussed below and used \deleted{further discussed} in Section \ref{sec:general_opt}.
    \item \deleted{Shared memory: }The \emph{shared memory}\deleted{ of NVIDIA GPUs} is a \deleted{fast and} low-latency memory space\replaced{,}{ that is} physically located on the GPU chip. The shared memory is used for data sharing within a block, \replaced{reducing memory access times}{which reduces the accessing} significantly compared with using the global memory. The access speed of the shared memory is roughly 100x faster than the uncached global memory. Other characteristics are that it is allocated for and shared between threads in the same block, and that its capacity is limited compared with the global memory. For example, NVIDIA Tesla V100 with compute capability 7.0, has \deleted{a limit of} 96KB \replaced{of}{for the} shared memory per SM and 80 SMs in total.
    \item \deleted{Registers: }The \replaced{\emph{register}}{ registers }    memory of NVIDIA GPUs is the fastest memory with the lowest latency. \replaced{Registers}{ These registers} are private to each thread and are not shared among other threads or blocks\replaced{. They}{, and they} are used for temporary data storage and computation in a single thread. The access\deleted{ing} speed for \deleted{the} registers is faster than the shared memory, but \replaced{not by orders of magnitude}{they are roughly comparable}. Register memory also has limited capacity. NVIDIA Tesla V100 has a limit of 64K 32-bit registers per SM.
\end{enumerate}

\subsubsection{Coalesced Memory Access}

An important feature of global memory \added{management} is coalesced memory access. \replaced{It is} {Coalesced memory access is} a memory access pattern that maximizes memory bandwidth and improves memory performance in parallel computing, particularly on GPUs. It refers to the way threads in a warp access memory in a contiguous and aligned manner, minimizing memory transactions and maximizing data transfer efficiency.

Depending on the architecture of the GPUs, the device can load or write the global memory via a 32-byte, 64-byte, or 128-byte transaction that is aligned with their sizes, respectively. If the threads in a warp access a continuous and aligned memory block with the size of 128-byte (this could be 32 4-byte single-precision floats, for example), it only costs one transaction. If the continuous memory block is misaligned with the 128-byte pattern, it costs two transactions. If the memory accessed is strided, which means no two elements are in the same 128-byte memory block, it will cost 32 transactions, even though it \replaced{accesses the same amount of memory}{equally accesses 128 bytes} as in the \replaced{coalesced}{continuous and aligned} case.


\subsection{GPU Kernel Design}

Our \deleted{implementation of the} GPU kernels are  enhanced and optimized versions of the code for unregularized OT that accompanies the paper~\cite{mai2021fast}, available on Github\footnote{The implementation is open-source and available at \href{https://github.com/vienmai/drot}{https://github.com/vienmai/drot}}. In the following subsections, we describe our enhancements of the basic kernel, and detail the extensions that we have made to implement quadratic and group-lasso regularization. We will also present run-time measurements to demonstrate that our kernels execute even faster than the "light-speed" per-iteration times of Sinkhorn. \deleted{ kernel run-times, illustrating that  the general optimizations that apply to both two main kernels of the regularized DROT and introduce the specific details of the two kernels, and the comparison on the per-iteration runtime.}

\subsubsection{General Optimization}\label{sec:general_opt}

Similar to the kernel described in \cite{mai2021fast}, the main kernel is designed in the way that each thread block contains 64 threads and is responsible for updating the sub-matrix with size $(64,64)$ in $X$. However, for small problem sizes, where the required number of blocks may be smaller than the number of SMs, 
this design may cause low utilization of the GPU, since not all of the SMs are utilized. To improve the utilization for small problem sizes, we adapt the number of columns assigned to each block to the problem size, \replaced{reducing the per-block work load and increasing the SM utilization. }{, which increases the required number of blocks to execute the kernel and makes fewer SMs idle.}

Accessing data in the global memory is critical to the performance of CUDA applications. In the kernel from~\cite{mai2021fast}, 
uncoalesced memory access happens when the number of rows is not a multiple of 32. For example, when the number of rows is 127 and the block tries to update the second column of $X$, the threads in the block will access 64 floats. However, the continuous block is not aligned with the 128-byte block structure, because the first element of the second column is in the same block as the last 31 elements in the first column. In this case, accessing 64 single-precision floats will require three 128-byte transactions while it could be two 128-byte transactions in an optimal solution. This can lead to strange situations where decreasing the problem size by one increases the per-iteration runtime increases significantly (6-20\%). To alleviate the problem and to make full use of the coalesced memory access, we updated the kernel from~\cite{mai2021fast} to handle shifted columns that match the 128-byte memory blocks. An illustration of the update is in Figure \ref{fig:GPUKernel} where the old working area of a block from  \cite{mai2021fast} is marked with "Planned Block Work Area" and the updated working area of a block is marked with "Actual Block Work Area". The comparison of the per-iteration runtime before and after the optimization in Table \ref{table:coalesced} shows that the update solves the problem of uncoalesced memory access, and that the implementation performs more predictably on various problem sizes.

\begin{table}[htbp]
\caption{Comparison of the per-iteration runtime before and after the update for coalesced memory access. Note that the focus here is the relative change in the per-iteration runtime after decreasing the problem size by one. The improvement in the per-iteration runtime under the same problem size is caused by other optimization techniques.}\label{table:coalesced}
\centering
\begin{tabular}{c|l|l}
\textbf{Problem Size} & \multicolumn{1}{c|}{\textbf{Runtime Before Update}} & \multicolumn{1}{c}{\textbf{Runtime After Update}} \\ \hline
$4096\times 4096$        & 0.3577                         & 0.3453                              \\
$4095\times 4095$        & 0.3812 (+6.58\%)               & 0.3417 (-1.04\%)                    \\\hline
$8192\times 8192$        & 1.281                         & 1.152                              \\
$8191\times 8191$        & 1.548  (+20.80\%)             & 1.153 (+0.10\%)                    \\\hline
$10240\times 10240$      & 1.916                         & 1.808                              \\
$10239\times 10239$      & 2.189 (+14.24\%)              & 1.781 (-1.52\%)  \\\hline                 
\end{tabular}
\end{table}

\subsubsection{Quadratically Regularized DROT}

Similar to the implementation in \cite{mai2021fast}, the matrices $X$ and $C$ are stored in the global memory in column-major order. For the problem sizes that we target, this is the only GPU memory that can fit the matrices. In addition, the performance penalty of storing them in a slower memory is limited since each element only needs to be loaded once for every update.\deleted{and it's because their sizes are large, and, more importantly, each element only needs to be loaded once for each update. It means that putting those two matrices in the shared memory is impractical and unnecessary.} 
Since the update of $X$ is done column by column, the auxillary variables
$\phi$ and $\varphi$
can be shared and reused among the threads in one block. To speedup the memory access, 
we load the required elements of $\phi$ and $\varphi$ into  shared memory at the beginning of the main kernel.

Figure \ref{fig:GPUKernel} 
illustrates how the update of one sub-column is done in a thread block of the \deleted{main kernel for the} quadratically regularized DROT \added{kernel}. The corresponding steps (1)-(5) in Figure \ref{fig:GPUKernel} are explained below.

\begin{enumerate}[(1)]
    \item Load the corresponding elements of $\phi$ and $\varphi$ from the shared memory.
    \item Load the corresponding elements of $X$ and $C$ from the global memory. Compute $[Y_{ij}-\rho C_{ij}]_+$ and $C_{ij}\cdot X_{ij}$. For the quadratically regularized case, we can directly update $X_{ij}$ with $Y_{ij}$ after multiplying it with the scale $1/(1+\rho \alpha)$.
    
    \item Perform in-warp reduction to compute
    the column sum and 
    the objective value in 
    the block.
    \item Use atomic operations to accumulate the results to the corresponding row and column sums.
    \item This step is only used by the group-lasso kernel, 
    introduced in Section \ref{sec:GLDROT}.
\end{enumerate}

Similar to the implementation by \cite{mai2021fast}, the iterations are divided into even iterations and odd iterations.
\begin{enumerate}
    \item For the even iterations, the kernel updates $X_{k+1}$ as
    $$
        X_{k+1} = \prox{\rho h}{[X_{k} + \phi_{k+1} \ones_n^\top + \ones_m \varphi_{k+1}^\top-\stepsize C]_+}-\stepsize C 
    $$
    \item For the odd iterations, it updates $X_{k+1}$ as
    $$
        X_{k+1} = \prox{\rho h}{[X_{k+1} + \phi_{k+1} \ones_n^\top + \ones_m \varphi_{k+1}^\top]_+}
    $$
\end{enumerate}

In this way, the kernel only needs to access the matrix $C$ once every second iteration.

Except for the general optimizations introduced in Section \ref{sec:general_opt}, the only difference from the basic kernel is  step (2) which additionally multiplies the result with a scalar. This minor change introduces negligible changes in the per-iteration runtime.

\subsubsection{Group-lasso Regularized DROT}\label{sec:GLDROT}

The kernel for group-lasso regularized DROT is slightly more involved, but also follows the five-step procedure illustrated in Figure~\ref{fig:GPUKernel}. More specifically, it performs the following steps:

\begin{enumerate}[(1)]
    \item Load the corresponding elements of $\phi$ and $\varphi$ from the shared memory.
    \item Load the corresponding elements of $X$ and $C$ from the global memory. Compute $[Y_{ij}-\rho C_{ij}]_+$ and $C_{ij}\cdot X_{ij}$. Compared with the quadratically regularized DROT, the kernel cannot update $X_{ij}$ until the scale of the group is ready in step $(5)$.
    \item Perform in-warp reduction. For the group-lasso regularized case, we also 
    reduce the sum of $[Y_{ij}-\rho C_{ij}]_+^2$ to compute 
    the norm of the group.
    \item Accumulate the results to the corresponding sum of the row, the sum of the column, and the sum of the squared elements within the group using atomic operations.
    \item The kernel can now 
    compute the scale for the group, and apply it to the computed $[Y_{ij}-\rho C_{ij}]_+$ so as to update $X_{ij}$. (For large problem sizes, where a thread block can only access part of the group, the process of reduction and broadcast is done in an additional kernel function).
\end{enumerate}

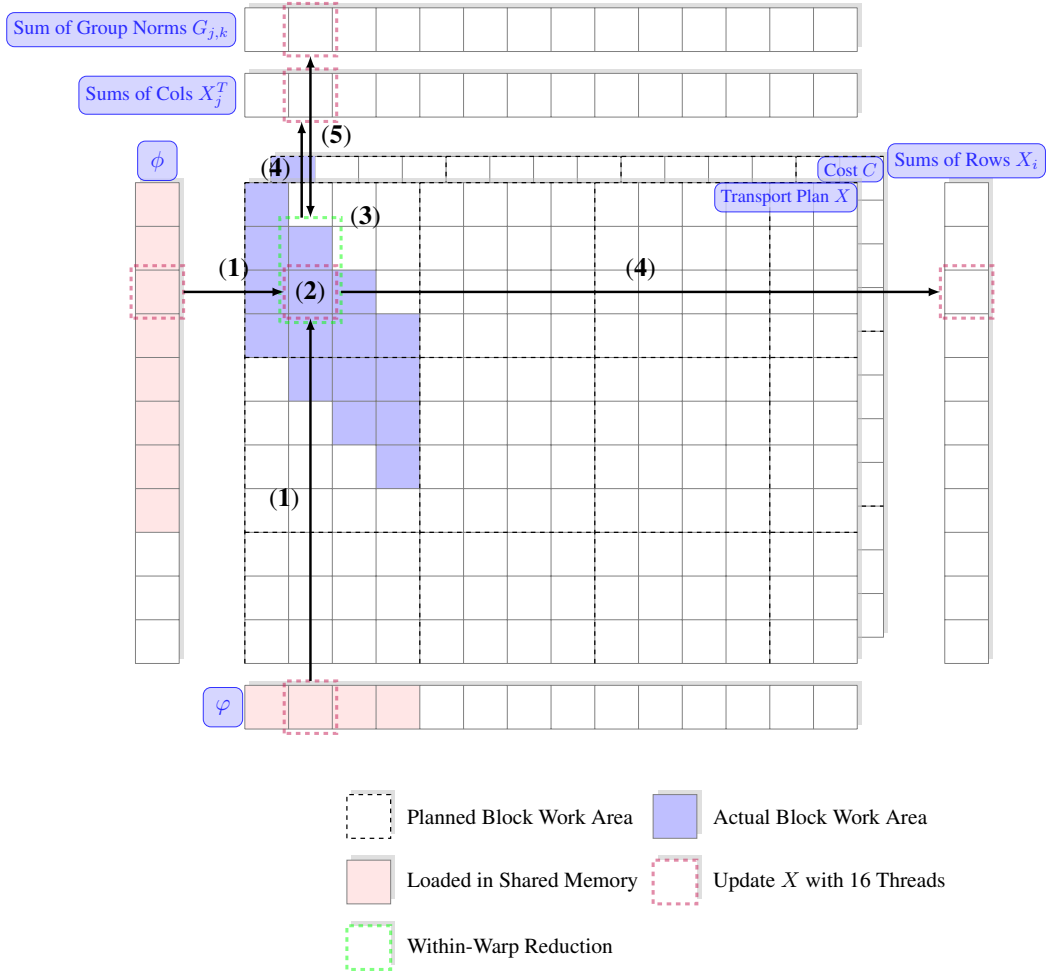
\begin{figure}[!htbp]
  \centering
  \resizebox{1\textwidth}{!}{
    \begin{tikzpicture}[every shadow/.style={shadow xshift=3pt, shadow yshift=3pt,fill=gray!150!black}]

        \newcommand{\NROWS}{11}
        \newcommand{\NCOLS}{14}
        \newcommand{\WORKX}{3}
        \newcommand{\WORKY}{2}
        \newcommand{\STEPTEXTSCALE}{1.7}
        
        \begin{scope}[xshift=0.6cm,yshift=0.6cm]
            \fill [white, drop shadow] (0,0) rectangle (\NCOLS,\NROWS);
            \draw [step=1cm,gray,very thin] (0,0) grid (\NCOLS,\NROWS);
            \node[below left, fill=blue!20, text=blue, draw=blue, scale=1.2, opacity=0.8, rounded corners=5pt] at (\NCOLS,\NROWS) {Cost $C$};

            \fill [blue!25] (0,\NROWS-4) rectangle (1,\NROWS);
            \fill [blue!25] (1,\NROWS-5) rectangle (2,\NROWS-1);
            \fill [blue!25] (2,\NROWS-6) rectangle (3,\NROWS-2);
            \fill [blue!25] (3,\NROWS-7) rectangle (4,\NROWS-3);

            \draw [step=4cm,thick, dashed, shift={(0,3)}] (0,-3) grid (\NCOLS,\NROWS-3);
        \end{scope}
        
        \begin{scope}[xshift=0cm,yshift=0cm]
            \fill [white, drop shadow] (0,0) rectangle (\NCOLS,\NROWS);

            \fill [blue!25] (0,\NROWS-4) rectangle (1,\NROWS);
            \fill [blue!25] (1,\NROWS-5) rectangle (2,\NROWS-1);
            \fill [blue!25] (2,\NROWS-6) rectangle (3,\NROWS-2);
            \fill [blue!25] (3,\NROWS-7) rectangle (4,\NROWS-3);
            
            \draw [step=1cm,gray,very thin] (0,0) grid (\NCOLS,\NROWS);

            \draw [step=4cm,thick, dashed, shift={(0,3)}] (0,-3) grid (\NCOLS,\NROWS-3);

            \node[below left, fill=blue!20, text=blue, draw=blue, scale=1.2, opacity=0.8, rounded corners=5pt] at (\NCOLS,\NROWS) {Transport Plan $X$};

            \draw[dashed ,opacity=0.5, line width=2pt, draw=purple!90] (\WORKY-1.1,\NROWS-\WORKX-0.1) rectangle (\WORKY+0.1,\NROWS-\WORKX+1.1);
            \node[scale=\STEPTEXTSCALE] at (\WORKY-0.5, \NROWS-\WORKX+0.5) {(\textbf{2})};

            \draw[dashed ,opacity=0.5, line width=2pt, draw=green] (\WORKY-1.2,\NROWS-\WORKX-0.2) rectangle (\WORKY+0.2,\NROWS-\WORKX+2.2);
            \node[anchor=west, scale=\STEPTEXTSCALE] at (\WORKY+0.2,\NROWS-\WORKX+2.2) {(\textbf{3})};

            \draw[-latex, ultra thick] (\WORKY-0.7,\NROWS-\WORKX+2.2) -- (\WORKY-0.7,\NROWS+1.4) node[midway, left, scale=\STEPTEXTSCALE] {(\textbf{4})};

            \draw[-latex, ultra thick] (\WORKY+0.2,\NROWS-\WORKX+0.5) -- (\NCOLS+1.9,\NROWS-\WORKX+0.5) node[midway, above, scale=\STEPTEXTSCALE] {(\textbf{4})};

        \end{scope}

        \begin{scope}[xshift=-2.5cm,yshift=0cm]
            \fill [white, drop shadow] (0,0) rectangle (1,\NROWS);
            \fill [pink!40] (0,\NROWS-8) rectangle (1,\NROWS);
            \draw [step=1cm,gray,very thin] (0,0) grid (1,\NROWS);
            \node[fill=blue!20, text=blue, draw=blue, scale=1.5, opacity=0.8, minimum width=0.6cm, minimum height=0.6cm, rounded corners=5pt] at (0.5,\NROWS+0.5) {$\phi$};

            \draw[dashed ,opacity=0.5, line width=2pt, draw=purple!90] (-0.1,\NROWS-\WORKX-0.1) rectangle (1.1,\NROWS-\WORKX+0.1+1);
            \draw[-latex, ultra thick] (1.1,\NROWS-\WORKX+0.5) -- (2.5+\WORKY-1-0.1,\NROWS-\WORKX+0.5) node[midway, above, scale=\STEPTEXTSCALE] {(\textbf{1})};
            
        \end{scope}

        \begin{scope}[xshift=0cm,yshift=-1.5cm]
            \fill [white, drop shadow] (0,0) rectangle (\NCOLS,1);
            \fill [pink!40] (0,0) rectangle (4,1);
            \draw [step=1cm,gray,very thin] (0,0) grid (\NCOLS,1);
            \node[fill=blue!20, text=blue, draw=blue, scale=1.5, opacity=0.8, minimum width=0.6cm, minimum height=0.6cm, rounded corners=5pt] at (-0.5,0.5) {$\varphi$};

            \draw[dashed ,opacity=0.5, line width=2pt, draw=purple!90] (\WORKY-0.1-1,-0.1) rectangle (\WORKY+0.1,1.1);
            \draw[-latex, ultra thick] (\WORKY-0.5,1.1) -- (\WORKY-0.5,1.5-0.1+\NROWS-\WORKX) node[midway, left, scale=\STEPTEXTSCALE] {(\textbf{1})};
        \end{scope}

        \begin{scope}[xshift=0cm,yshift=\NROWS cm+1.5cm]
            \fill [white, drop shadow] (0,0) rectangle (\NCOLS,1);
            \draw [step=1cm,gray,very thin] (0,0) grid (\NCOLS,1);
            \node[fill=blue!20, text=blue, draw=blue, scale=1.3, opacity=0.8, minimum width=0.6cm, minimum height=0.6cm, anchor=east, rounded corners=5pt] at (-0.2,0.5) {Sums of Cols $X^T_{j}$};

            \draw[dashed ,opacity=0.5, line width=2pt, draw=purple!90] (\WORKY-1.1,-0.1) rectangle (\WORKY+0.1,1.1);
        \end{scope}

        \begin{scope}[xshift=\NCOLS cm+2 cm,yshift=0cm]
            \fill [white, drop shadow] (0,0) rectangle (1,\NROWS);
            \draw [step=1cm,gray,very thin] (0,0) grid (1,\NROWS);
            \node[fill=blue!20, text=blue, draw=blue, scale=1.3, opacity=0.8, minimum width=0.6cm, minimum height=0.6cm, rounded corners=5pt] at (0.5,\NROWS+0.5) {Sums of Rows $X_{i}$};

            \draw[dashed ,opacity=0.5, line width=2pt, draw=purple!90] (-0.1,\NROWS-\WORKX-0.1) rectangle (1.1,\NROWS-\WORKX+1.1);
        \end{scope}

        \begin{scope}[xshift=0cm,yshift=\NROWS cm+3cm]
            \fill [white, drop shadow] (0,0) rectangle (\NCOLS,1);
            \draw [step=1cm,gray,very thin] (0,0) grid (\NCOLS,1);
            \node[fill=blue!20, text=blue, draw=blue, scale=1.3, opacity=0.8, minimum width=0.6cm, minimum height=0.6cm, anchor=east, rounded corners=5pt] at (-0.2,0.5) {Sum of Group Norms $G_{j,k}$};

            \draw[dashed ,opacity=0.5, line width=2pt, draw=purple!90] (\WORKY-1.1,-0.1) rectangle (\WORKY+0.1,1.1);

            \draw[latex-latex, ultra thick] (1.5,-0.1) -- (1.5, -3.8) node[midway, right, scale=\STEPTEXTSCALE] {(\textbf{5})};
        \end{scope}

        \begin{scope}[xshift=0cm,yshift=-4 cm]
            \fill [white, drop shadow] (\NCOLS/6,0) rectangle (\NCOLS/6+1,1);
            \draw [thick, dashed] (\NCOLS/6,0) 
            rectangle (\NCOLS/6+1,1);
            \node [anchor=west, scale=1.4] at (\NCOLS/6+1.2,0.5) {Planned Block Work Area};

            \fill [blue!25, drop shadow] (\NCOLS*4/6,0) rectangle (\NCOLS*4/6+1,1);
            \draw [gray, very thin] (\NCOLS*4/6,0) 
            rectangle (\NCOLS*4/6+1,1);
            \node [anchor=west, scale=1.4] at (\NCOLS*4/6+1.2,0.5) {Actual Block Work Area};
        \end{scope}

        \begin{scope}[xshift=0cm,yshift=-5.5cm]
            \fill [pink!40, drop shadow] (\NCOLS/6,0) rectangle (\NCOLS/6+1,1);
            \draw [gray, very thin] (\NCOLS/6,0) 
            rectangle (\NCOLS/6+1,1);
            \node [anchor=west, scale=1.4] at (\NCOLS/6+1.2,0.5) {Loaded in Shared Memory};

            \fill [white!25, drop shadow] (\NCOLS*4/6,0) rectangle (\NCOLS*4/6+1,1);
            \draw [dashed ,opacity=0.5, line width=2pt, draw=purple!90] (\NCOLS*4/6-0,0) 
            rectangle (\NCOLS*4/6+1,1);
            \node [anchor=west, scale=1.4] at (\NCOLS*4/6+1.2,0.5) {Update $X$ with 16 Threads};
        \end{scope}

        \begin{scope}[xshift=0cm,yshift=-7cm]
            \fill [white!40, drop shadow] (\NCOLS/6,0) rectangle (\NCOLS/6+1,1);
            \draw [dashed ,opacity=0.5, line width=2pt, draw=green] (\NCOLS/6,0) 
            rectangle (\NCOLS/6+1,1);
            \node [anchor=west, scale=1.4] at (\NCOLS/6+1.2,0.5) {Within-Warp Reduction};

        \end{scope}

    
    \end{tikzpicture}
  }
  \caption{GPU Kernel Design}\label{fig:GPUKernel}
\end{figure}

Compared with the quadratically regularized case, the group-lasso regularized DROT requires an additional round for reduction and broadcast within the group. Moreover, for large problem sizes, when the intermediate results for the group do not fit in the shared memory, it is inevitable to access the matrix $C$ in every iteration.

Both factors lead to an increase in the per-iteration runtime compared with the quadratically regularized DROT, and the runtime per iteration for GLDROT is around 2x of QDROT's. However, the per-iteration runtime of GLDROT is still comparable with the Sinkhorn-Knopp algorithm, which will be detailed in the next section.



\subsubsection{Per-iteration Runtime}\label{sec:per_iter_runtime}

The comparison of the per-iteration runtime among all methods is given in Figure \ref{fig:per_iter_cmp}. For QDROT and the Sinkhorn-Knopp algorithm, 12 random samples are generated with problem sizes from 100 to 10000. For GLDROT, 12 random samples are generated with problem sizes from 100 to 10000 and with the number of classes as 2 and 4, respectively. Note that the Sinkhorn-based implementation of group-lasso OT  in POT \cite{flamary2021pot} does not return the number of iterations, nor reports the per-iteration run-time. It is therefore impossible to estimate the corresponding run-time numbers for the Sinkhorn-based group-lasso OT \cite{cuturi2013sinkhorn}. We have therefore excluded this algorithm here, but expect it to have larger (and possibly significantly larger) per-iteration runtime than the Sinkhorn-Knopp algorithm \cite{cuturi2013sinkhorn}.


To make a fair comparison between the Sinkhorn-Knopp algorithm and ours, we cover four variants of the Sinkhorn-Knopp algorithm in the comparison. In the implementation of POT \cite{flamary2021pot}, it computes the primal residual once for every 10 iterations because the operation is relatively costly compared with the iteration itself. Moreover, to increase the numerical stability of the algorithm and to improve the quality of the results with smaller weights of the entropic regularizer, the Sinkhorn-Knopp algorithm is improved by moving the computation into log scale. By alternating whether computing the primal residual and computing in log scale, we present four versions of the Sinkhorn-Knopp algorithm.

As for our methods, due to the design that RDROT is working directly on $X$, the computation of the objective value and the primal residual introduces almost no increase in the per-iteration runtime.

From the results, it shows that our QDROT implementation has roughly the same performance as the Sinkhorn-Knopp algorithm from the perspective of the per-iteration runtime. The per-iteration runtime of the Sinkhorn-Knopp algorithm in log scale is roughly 4.3-9.4x longer than QDROT and is roughly 2.5-4.5x longer than GLDROT. If comparing GLDROT with different numbers of classes, it shows our implementation has comparable computation complexity even if the number of classes increases.

\begin{figure}[!htbp]
    \centering
    \includegraphics[width=0.7\textwidth]{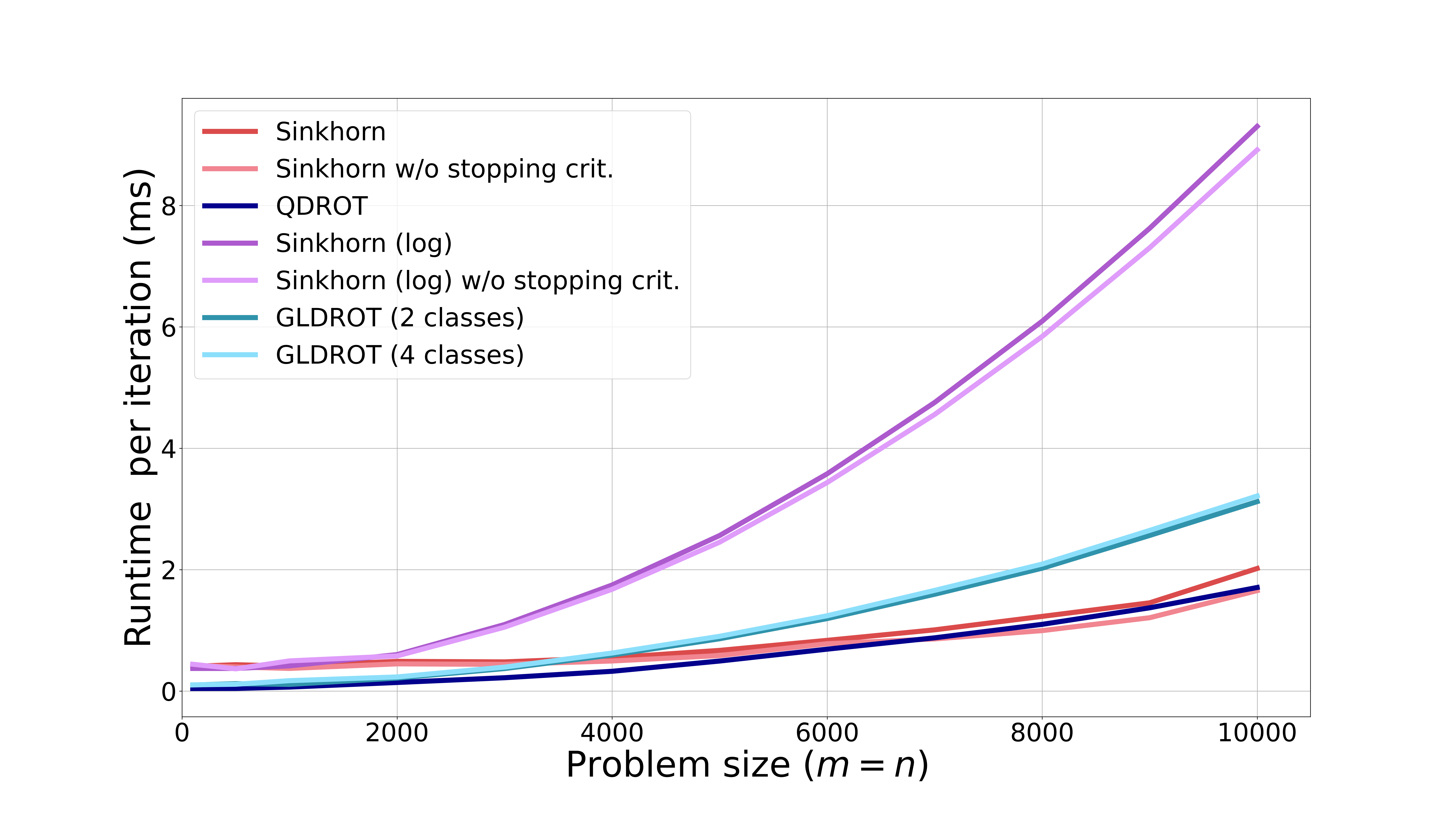}
    \caption{Comparison of the per-iteration cost for OT solvers. We compare the Sinkhorn-Knopp Algorithm \protect\cite{cuturi2013sinkhorn}, and a log domain variant (with and without stopping criterion computations every 10th iteration), against QDROT (ours): Quadratically regularized DROT, and GLDROT (ours): Group-lasso regularized DROT.}
    \label{fig:per_iter_cmp}
\end{figure}

\section{Additional Experiments}
To strengthen the numerical results of the paper, we added additional benchmarks of RDROT with the quadratic and group-lasso regularization with more datasets of different sizes and hyperparameters.
\paragraph{Quadratic regularization}
In the paper, we compared our solver with an L-BFGS method applied to the dual problem. Two problem sizes were considered: $m=1000$, $n=1000$, and $m=2000$, $n=3000$. In Figure~\ref{fig:quad:ot:benchmark:sup}, we rerun the benchmark on 4 additional problem sizes. Notice that all results are consistent with those presented in the paper, showing that our results generalize to more settings beyond the ones presented in the paper. 
\begin{figure}[!htbp]
    \centering
    \begin{subfigure}[b]{0.4\textwidth}\label{fig:quad100}
        \centering
        {\includegraphics[width=1.\textwidth]{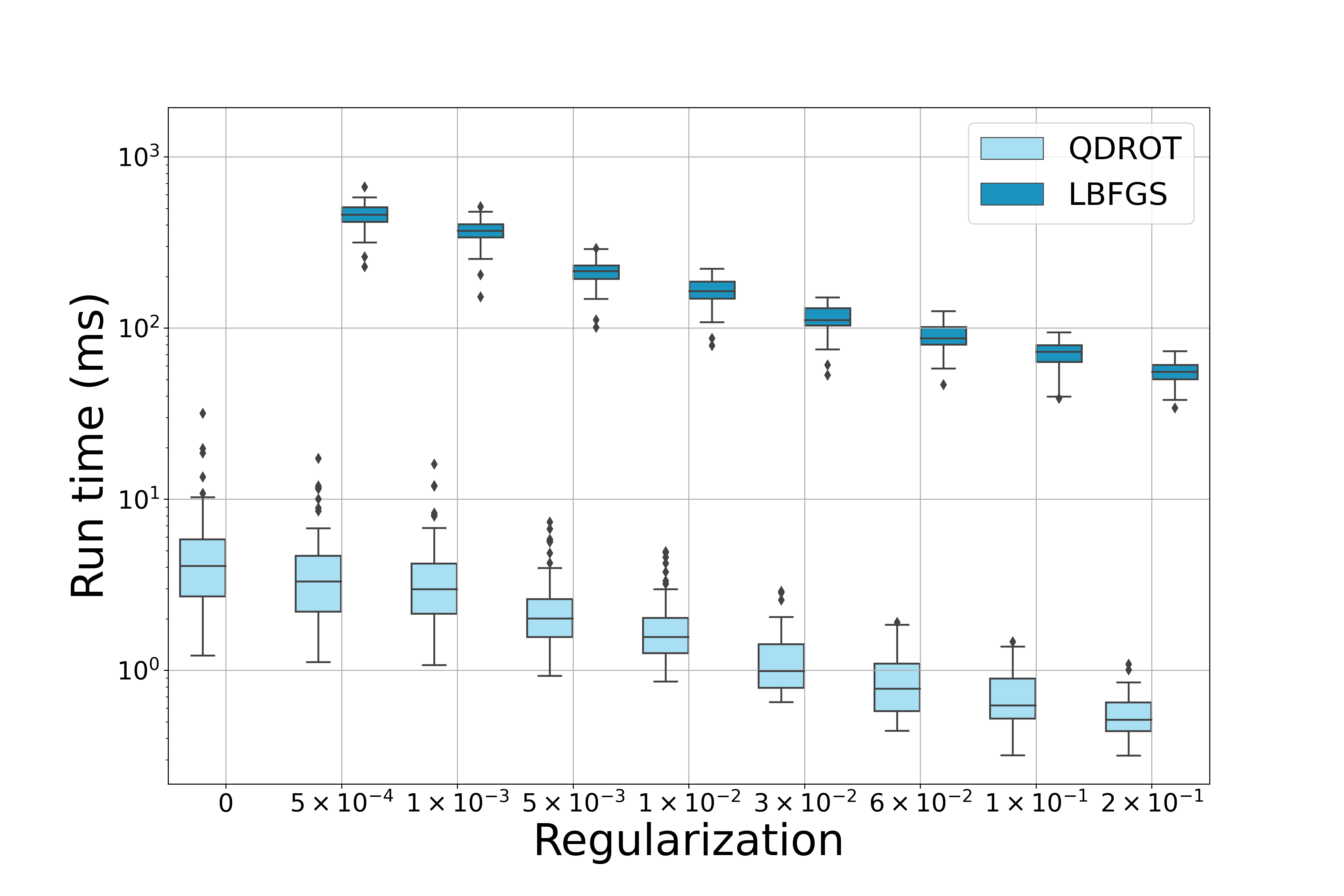}}
        \caption{$m=100$, $n=100$}
    \end{subfigure}
    ~
    \begin{subfigure}[b]{0.4\textwidth}\label{fig:quad500}
        \centering
        {\includegraphics[width=1.\textwidth]{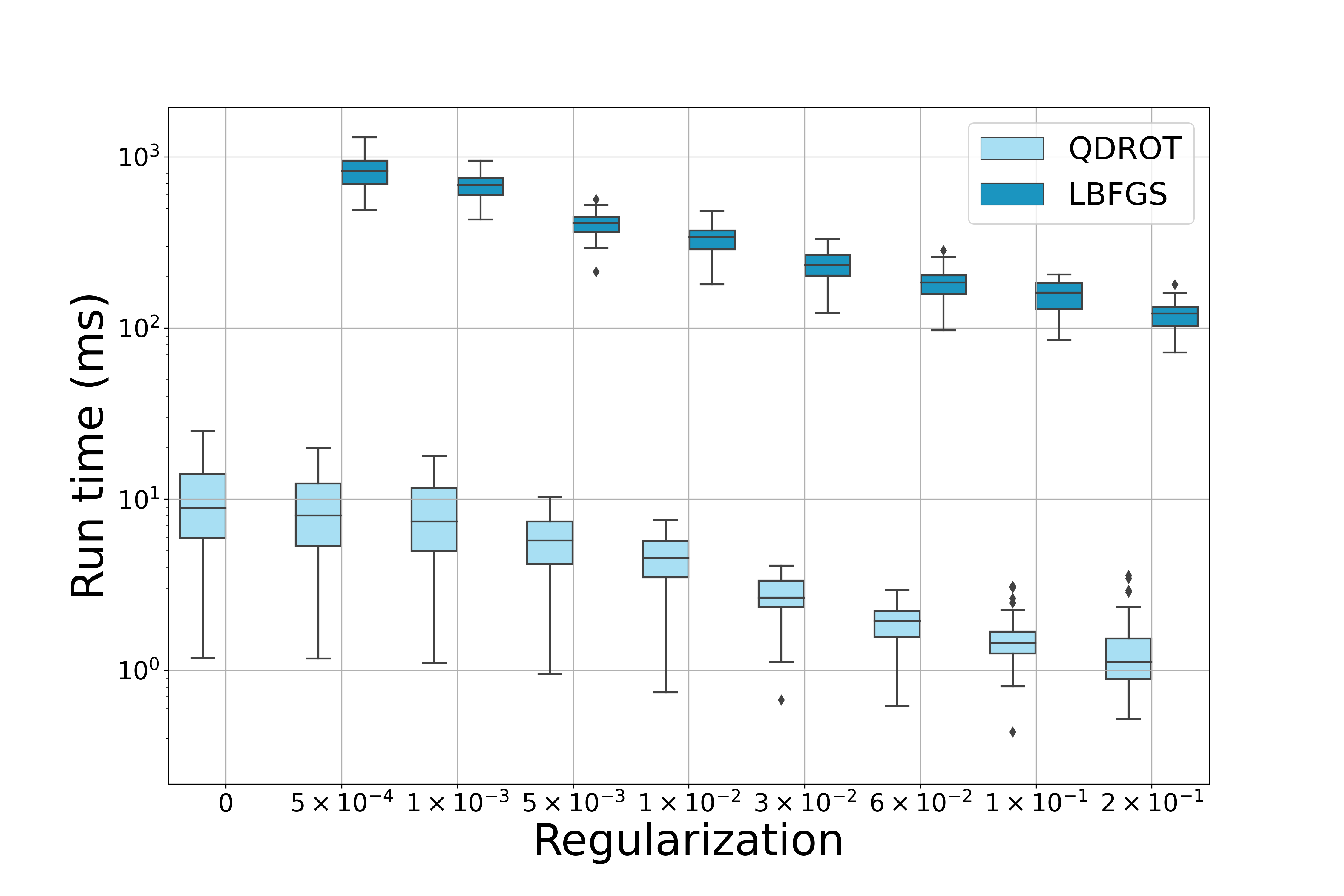}}
        \caption{$m=500$, $n=500$}
    \end{subfigure}
    \\
    \begin{subfigure}[b]{0.4\textwidth}\label{fig:quad5000}
        \centering
        {\includegraphics[width=1.\textwidth]{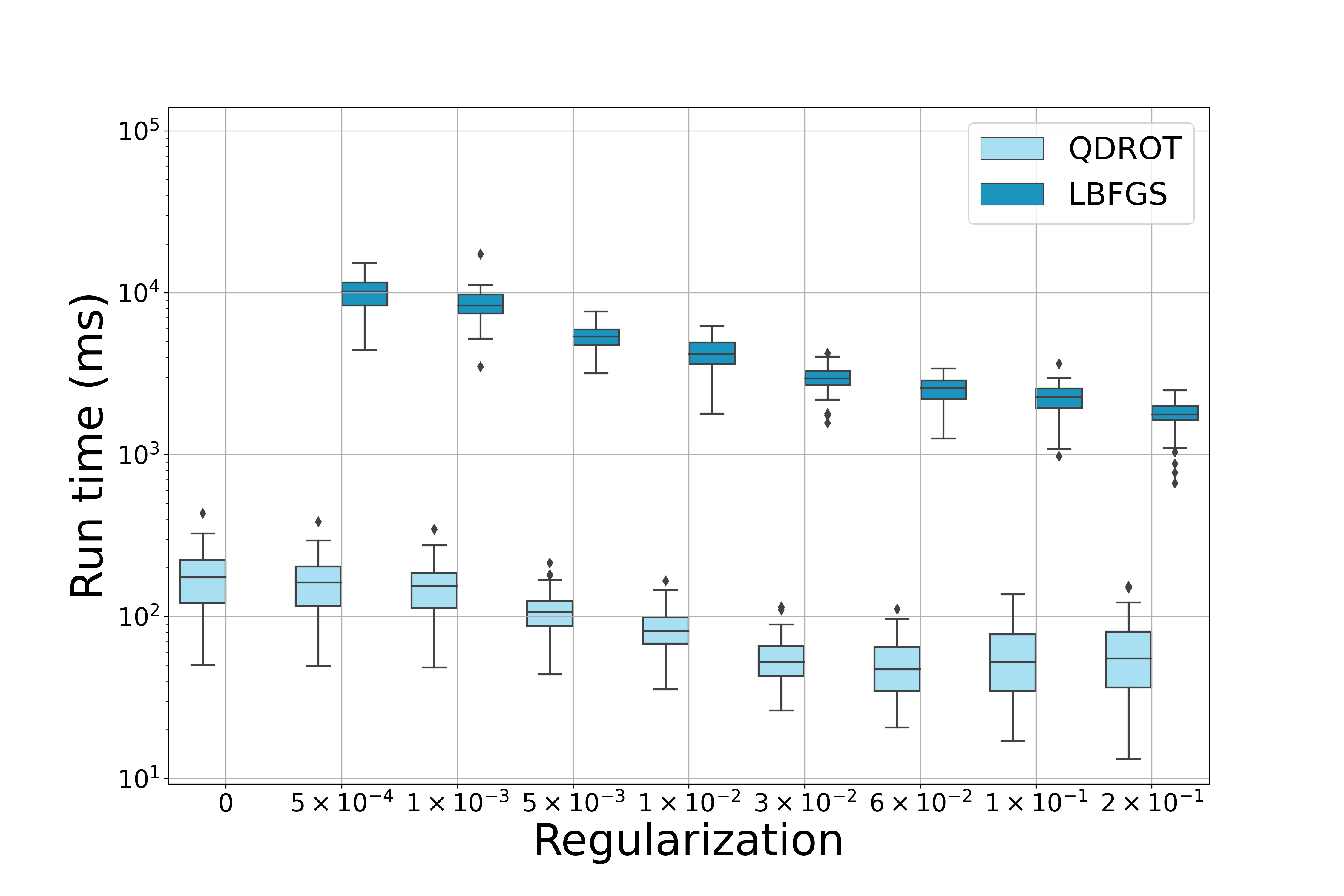}}
        \caption{$m=5000$, $n=5000$}
    \end{subfigure}
    ~
    \begin{subfigure}[b]{0.4\textwidth}\label{fig:quad10000}
        \centering
        {\includegraphics[width=1.\textwidth]{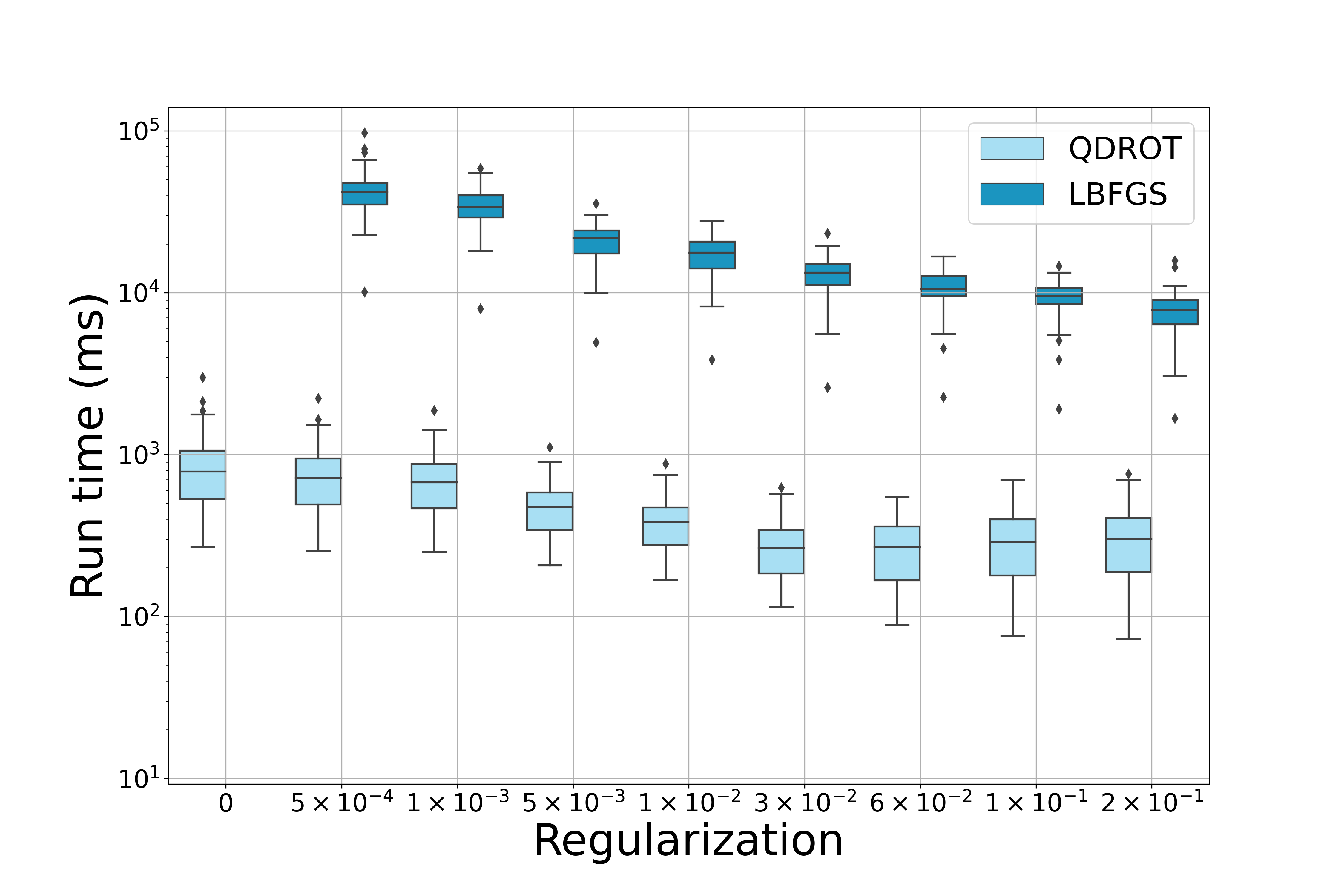}}
        \caption{$m=10000$, $n=10000$}
    \end{subfigure}
    \caption{Additional experiments of the RDROT (QDROT) and the dual L-BFGS method comparison for different quadratic regularization parameters and problem sizes. 50 datasets of an additional 4 problem sizes were simulated.}
    \label{fig:quad:ot:benchmark:sup}
\end{figure}

\paragraph{Group Lasso regularization}
To establish the computational advantages of RDROT applied to the Group-Lasso problem, we benchmarked it against Sinkhorn-based solver on simulated datasets of size $m=1000$, $n=1500$, for a selection of hyperparameters. It Table~{\ref{fig:gl:500},\ref{fig:gl:1000},\ref{fig:gl:1500}}, we rerun the benchmark on more problem sizes and hyperparameters, to test if the results generalize. Indeed, GLDROT consistently outperforms the state-of-the-art in all considered problems.

\begin{table}[!htbp]
\centering
\caption{GL experiments with $n=500$ and $m=500$}\label{fig:gl:500}
\begin{tabular}{l|c c| c c c | c c c}
    \hline
      & \multicolumn{2}{c|}{Reg.}  &\multicolumn{3}{c|}{Runtime (s) $\downarrow$}    & \multicolumn{3}{c}{Agg. $W_2$ dist. $\downarrow$}  \\
       Method & Ent & GL & Median  & $q10$ & $q90$ & Median  & $q10$ & $q90$ \\\hline
        GLSK & \texttt{1e-3} & \texttt{1e-6} & 1.26 & 1.24 & 2.98 & 0.35 & 0.084 & 4.96 \\ 
         & \texttt{1e-3} & \texttt{5e-4} & 1.23 & 1.22 & 2.93 & 0.351 & 0.0841 & 4.96 \\
         & \texttt{1e-3} & \texttt{5e-2} & 1.35 & 1.23 & 3.57 & 0.377 & 0.092 & 4.96 \\ 
         & \texttt{1e-2} & \texttt{1e-6} & 1.21 & 1.09 & 2.63 & 1.22 & 0.668 & 5.48 \\ 
         & \texttt{1e-2} & \texttt{5e-4} & 1.18 & 1.07 & 2.61 & 1.22 & 0.668 & 5.48 \\ \
         & \texttt{1e-2} & \texttt{5e-2} & 1.2 & 1.08 & 2.63 & 1.22 & 0.668 & 5.57 \\ 
         & \texttt{1e-1} & \texttt{1e-6} & 1.04 & 1.04 & 1.06 & 7.83 & 3.27 & 30.6 \\ 
         & \texttt{1e-1} & \texttt{5e-4} & 1.05 & 1.05 & 1.07 & 7.83 & 3.27 & 30.6 \\ 
         & \texttt{1e-1} & \texttt{5e-2} & 1.05 & 1.04 & 1.06 & 7.83 & 3.27 & 30.6 \\ 
         & \texttt{1.}& \texttt{1e-6} & 1.03 & 0.517 & 1.04 & 45.3 & 9.74 & 212 \\ 
         & \texttt{1.}& \texttt{5e-4} & 1.02 & 0.512 & 1.03 & 45.3 & 9.74 & 212 \\ 
         & \texttt{1.}& \texttt{5e-2} & 1.03 & 0.514 & 1.03 & 45.3 & 9.74 & 212 \\ \hline
        GLDROT & & \texttt{1e-6}  & 0.0924 & 0.0519 & 0.16 & 0.0618 & 0.028 & 0.309 \\ 
        & & \texttt{5e-4}  & 0.0553 & 0.0354 & 0.106 & 0.0801 & 0.0357 & 0.319 \\ 
        & & \texttt{5e-2}  & 0.0127 & 0.00889 & 0.0351 & 0.672 & 0.321 & 2.87 \\\hline
    \end{tabular}
\end{table}

\begin{table}[!htbp]
    \centering
    \caption{$m=1000$, $n=1000$}\label{fig:gl:1000}
    \begin{tabular}{l|c c| c c c | c c c}
    \hline
      & \multicolumn{2}{c|}{Reg.}  &\multicolumn{3}{c|}{Runtime (s) $\downarrow$}    & \multicolumn{3}{c}{Agg. $W_2$ dist. $\downarrow$}  \\
       Method & Ent & GL & Median  & $q10$ & $q90$ & Median  & $q10$ & $q90$ \\\hline
        GLSK & \texttt{1e-3} &\texttt{1e-6}& 4.97 & 4.91 & 9.27 & 1.66 & 0.405 & 11 \\ 
        & \texttt{1e-3} &\texttt{5e-4}& 4.9 & 4.82 & 9.19 & 1.66 & 0.405 & 11 \\ 
        & \texttt{1e-3} &\texttt{5e-2}& 5.52 & 4.91 & 15.9 & 1.64 & 0.389 & 11 \\ 
        & \texttt{1e-2} &\texttt{1e-6}& 4.9 & 4.79 & 5.53 & 3.06 & 0.926 & 8.97 \\ 
        & \texttt{1e-2} &\texttt{5e-4}& 4.79 & 4.66 & 5.5 & 3.06 & 0.926 & 8.97 \\ 
        & \texttt{1e-2} &\texttt{5e-2}& 4.88 & 4.73 & 5.55 & 3.15 & 0.926 & 9.03 \\ 
        & \texttt{1e-1}  &\texttt{1e-6}& 4.71 & 4.68 & 4.73 & 20.6 & 8.97 & 74.3 \\ 
        & \texttt{1e-1} &\texttt{5e-4}& 4.74 & 4.7 & 4.78 & 20.6 & 8.97 & 74.3 \\ 
        & \texttt{1e-1} &\texttt{5e-2}& 4.76 & 4.73 & 4.82 & 20.6 & 8.97 & 74.3 \\ 
        & \texttt{1.} &\texttt{1e-6}& 4.61 & 4.57 & 4.64 & 130 & 37.5 & 403 \\ 
        &\texttt{1.}&\texttt{5e-4}& 4.62 & 4.6 & 4.67 & 130 & 37.5 & 403 \\ 
        &\texttt{1.}&\texttt{5e-2}& 4.74 & 4.71 & 4.77 & 130 & 37.5 & 403 \\ \hline
        GLDROT & & \texttt{1e-6}&  0.109 & 0.0735 & 0.146 & 0.0999 & 0.0424 & 0.304 \\ 
        & & \texttt{5e-4}&  0.0598 & 0.045 & 0.0763 & 0.138 & 0.0632 & 0.44\\ 
        & & \texttt{5e-2}&  0.0281 & 0.0232 & 0.0374 & 3.16 & 1.14 & 9.01 \\ \hline
    \end{tabular}
\end{table}

\begin{table}[!htbp]
    \centering
    \caption{$m=1000$, $n=1500$}\label{fig:gl:1500}
    \begin{tabular}{l|c c| c c c | c c c}
    \hline
      & \multicolumn{2}{c|}{Reg.}  &\multicolumn{3}{c|}{Runtime (s) $\downarrow$}    & \multicolumn{3}{c}{Agg. $W_2$ dist. $\downarrow$}  \\
       Method & Ent & GL & Median  & $q10$ & $q90$ & Median  & $q10$ & $q90$ \\\hline
        GLSK & \texttt{1e-3} &\texttt{1e-6}& 3.82 & 3.75 & 9.03 & 0.311 & 0.0657 & 6.48 \\
        & \texttt{1e-3} &\texttt{5e-4}& 3.8 & 3.75 & 9 & 0.311 & 0.0657 & 6.48 \\ 
        & \texttt{1e-3} &\texttt{5e-2}& 3.8 & 3.77 & 10.6 & 0.345 & 0.0665 & 6.48 \\  
        & \texttt{1e-2} &\texttt{1e-6}& 3.78 & 3.73 & 6.8 & 1.21 & 0.69 & 5.47 \\ 
        & \texttt{1e-2} &\texttt{5e-4}& 3.8 & 3.76 & 6.82 & 1.21 & 0.69 & 5.47 \\
        & \texttt{1e-2} &\texttt{5e-2}& 3.82 & 3.79 & 6.88 & 1.21 & 0.69 & 5.47 \\ 
        & \texttt{1e-1} &\texttt{1e-6}& 3.7 & 3.67 & 3.72 & 8.24 & 3.47 & 31.6 \\
        & \texttt{1e-1} &\texttt{5e-4}& 3.75 & 3.73 & 3.78 & 8.24 & 3.47 & 31.6 \\ 
        & \texttt{1e-1} &\texttt{5e-2}& 3.77 & 3.73 & 3.81 & 8.24 & 3.47 & 31.6 \\ 
        &\texttt{1.}&\texttt{1e-6}& 3.73 & 1.86 & 3.77 & 45.5 & 9.92 & 218 \\
        &\texttt{1.}&\texttt{5e-4}& 3.69 & 1.85 & 3.73 & 45.5 & 9.92 & 218 \\ 
        &\texttt{1.}&\texttt{5e-2}& 3.75 & 1.88 & 3.78 & 45.5 & 9.92 & 218 \\ \hline
        GLDROT && \texttt{1e-6}&  0.113 & 0.0758 & 0.147 & 0.0475 & 0.0215 & 0.137 \\ 
        && \texttt{5e-4}&  0.0745 & 0.0549 & 0.0951 & 0.0529 & 0.0245 & 0.163 \\ 
        && \texttt{5e-2}& 0.0232 & 0.0178 & 0.0288 & 0.331 & 0.154 & 1.38 \\ \hline
    \end{tabular}
\end{table}
\vspace{10cm}

\section{Model specification of the generative Adversarial Model}\label{appendix:additional:GAN}
We adopted similar network structures and the same loss function as in \cite{salimans2018improving} and performed several experiments with image generation based on the MNIST and CIFAR10 datasets.

To define a loss, we need  the following sample based OT-cost. 
\begin{align*}
\mathcal{W}_{c,h}(\mathbf{X},\mathbf{Y}) &=\mathrm{OT}_h(C_{\mathbf{X},\mathbf{Y}}) \\&:= \inf_{M \in \R_+^{m\times n}}\bigg\{\InP{C_{\mathbf{X}, \mathbf{Y}}}{M}  + h(M): M\ones_n = m^{-1} \ones_m, \, M^\top \ones_m = n^{-1} \ones_n \bigg\}.
\end{align*}
Here $m$ and $n$ are the batch sizes corresponding to $\mathbf{X}$ and $\mathbf{Y}$,  and $C_{\mathbf{X}, \mathbf{Y}}$ is a matrix with pairwise distances between the samples of $\mathbf{X}$ and $ \mathbf{Y}$. This can used to derive the mini-batch energy distance \cite{salimans2018improving}, defined as
\begin{align*}
\mathcal{L}_h=\mathcal{W}_{c,h}(\mathbf{X},\mathbf{Y})&+\mathcal{W}_{c,h}(\mathbf{X}',\mathbf{Y})+\mathcal{W}_{c,h}(\mathbf{X},\mathbf{Y}')+\mathcal{W}_{c,h}(\mathbf{X}',\mathbf{Y}')\\
&-2\mathcal{W}_{c,h}(\mathbf{X},\mathbf{X}')-2\mathcal{W}_{c,h}(\mathbf{Y},\mathbf{Y}').
\end{align*}
$\mathbf{X}$, $\mathbf{X}'$ are two independent mini-batches from real data, while $\mathbf{Y}$, $\mathbf{Y}'$ are two independent mini-batches generated from the generator.

In the experiments, we used the cosine similarity to parameterize the cost, and we replaced the OT solver used in~\cite{salimans2018improving} - the Sinkhorn-Knopp algorithm, by the PyTorch wrapper for our OT solver with the quadratic regularizer.

The model structures are adapted from the ones used in \cite{salimans2018improving}. Weight normalization is used to construct the parameters of the models \cite{salimans2016weight}, and the activation functions are gated linear units \cite{dauphin2017language} and concatenated ReLUs \cite{shang2016understanding}. We trained the model with the Adam optimzier \cite{kingma2014adam} using an initial learning rate of $3\times 10^{-4}$, $\beta_1=0.5$ and $\beta_2=0.999$. The batch size is 1024 for the MNIST experiment and 2048 for the CIFAR10 experiment. We update the generator three times for every discriminator update. As for the parameters of our OT solver, we use $\epsilon={10}^{-4}$ as the stopping criterion and $\lambda=10^{-3}$ as the weight of the quadratic regularizer.

Different combinations of hyperparameters, model structures and schedules for the weight of the regularizers and the stopping criterion (the OT solver can allowed to produce low-accuracy solutions at the beginning of the training, for example) may potentially improve the results of the generated samples, but since GANs are not the focus of this paper, we leave such refinements as future work.

\newpage
\subsection{MNIST Architecture and Results}

\begin{table*}[h]
  \centering
    \begin{tabular}{c|c|c|c|c|c}
    \hline
     Opertion & Activation & Kernel Size & Stride & Padding & Output Shape \\\hline

    Sample $z$ & & & & & $[4]$ \\
    Linear & GLU & & & & $[128\cdot7\cdot7]$ \\
    Reshape &  & & & & $[128, 7, 7]$ \\
    Upsample $\times 2$ & & & & & $[128, 14, 14]$ \\
    2D Convolution & GLU & $[5,5]$ & 1 & same & $[128, 14, 14]$ \\
    Upsample $\times 2$ & & & & & $[128, 28, 28]$ \\
    2D Convolution & GLU & $[5,5]$ & 1 & same & $[64, 28, 28]$ \\
    2D Convolution & tanh & $[5,5]$ & 1 & same & $[1, 28, 28]$ \\
    \hline
    \end{tabular}
    \caption{Generator Architecture for MNIST}
\end{table*}

\begin{table*}[h]
  \centering
    \begin{tabular}{c|c|c|c|c|c}
    \hline
     Opertion & Activation & Kernel Size & Stride & Padding & Output Shape \\\hline

    2D Convolution & CReLU & $[5,5]$ & 1 & same & $[128, 28, 28]$ \\
    2D Convolution & CReLU & $[5,5]$ & 1 & same & $[128, 28, 28]$ \\
    2D Convolution & CReLU & $[5,5]$ & 2 & 2 & $[256, 14, 14]$ \\
    2D Convolution & CReLU & $[5,5]$ & 2 & 2 & $[256, 7, 7]$ \\
    Flatten & & & & & $[256\cdot 7 \cdot 7]$ \\
    L2 Normalization & & & & & $[256\cdot 7 \cdot 7]$ \\
    \hline
    \end{tabular}
    \caption{Discriminator Architecture for MNIST}
\end{table*}

\begin{figure*}[h]
    \centering
    \begin{minipage}{0.65\textwidth}
        \centering
        {\includegraphics[width=\textwidth]{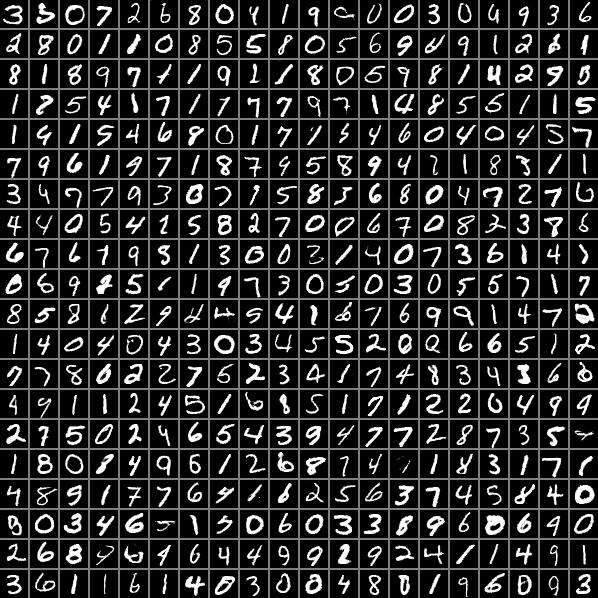}}
        \caption{Generated samples at 100th epoch for MNIST (Original size for each sample is 28x28 px)}\label{fig:mnist_gan}
    \end{minipage}
\end{figure*}

The generated samples in Figure \ref{fig:mnist_gan} are clear and shape, and no mode collapse can be observed (digits from 0 to 9 with various variants can be found in the figure). Compared with the results in \cite{genevay2018learning} (Figure 3), our results show higher quality.

\newpage
\subsection{CIFAR10 Architecture and Results}

\begin{table}[htbp]
  \centering
    \begin{tabular}{c|c|c|c|c|c}
    \hline
     Opertion & Activation & Kernel Size & Stride & Padding & Output Shape \\\hline

    Sample $z$ & & & & & $[100]$ \\
    Linear & GLU & & & & $[1024\cdot4\cdot4]$ \\
    Reshape &  & & & & $[1024, 4, 4]$ \\
    Upsample $\times 2$ & & & & & $[1024, 8, 8]$ \\
    2D Convolution & GLU & $[5,5]$ & 1 & same & $[512, 8, 8]$ \\
    Upsample $\times 2$ & & & & & $[512, 16, 16]$ \\
    2D Convolution & GLU & $[5,5]$ & 1 & same & $[256, 16, 16]$ \\
    Upsample $\times 2$ & & & & & $[256, 32, 32]$ \\
    2D Convolution & GLU & $[5,5]$ & 1 & same & $[128, 32, 32]$ \\
    2D Convolution & tanh & $[5,5]$ & 1 & same & $[3, 32, 32]$ \\
    \hline
    \end{tabular}
    \caption{Generator Architecture for CIFAR10}
\end{table}

\begin{table}[htbp]
  \centering
    \begin{tabular}{c|c|c|c|c|c}
    \hline
     Opertion & Activation & Kernel Size & Stride & Padding & Output Shape \\\hline
    2D Convolution & CReLU & $[5,5]$ & 1 & same & $[256, 32, 32]$ \\
    2D Convolution & CReLU & $[5,5]$ & 2 & 2 & $[512, 16, 16]$ \\
    2D Convolution & CReLU & $[5,5]$ & 2 & 2 & $[1024, 8, 8]$ \\
    2D Convolution & CReLU & $[5,5]$ & 2 & 2 & $[2048, 4, 4]$ \\
    Flatten & & & & & $[2048\cdot 4 \cdot 4]$ \\
    L2 Normalization & & & & & $[2048\cdot 4 \cdot 4]$ \\
    \hline
    \end{tabular}
    \caption{Discriminator Architecture for CIFAR10}
\end{table}

\begin{figure*}[!ht]
    \centering
    \begin{minipage}{0.65\textwidth}
        \centering
        {\includegraphics[width=\textwidth]{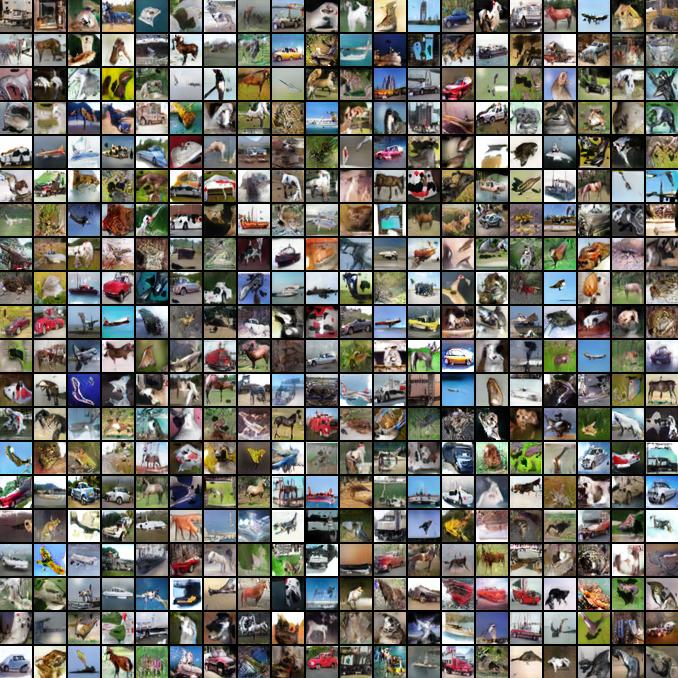}}
        \caption{Generated samples at 1900th epoch for CIFAR10 (Original size for each sample is 32x32 px)\label{fig:cifar10_gan}}
    \end{minipage}
\end{figure*}

In Figure \ref{fig:cifar10_gan}, most samples are recognizable, and the samples cover the 10 categories of the CIFAR10 dataset with enough diversity. Compared with the results in \cite{salimans2018improving} (Figure 4) which uses the Sinkhorn-Knopp algorithm as the OT solver in the mini-batch energy loss, our samples have similar quality.

\end{document}